%% file: blindfl.tex
\documentclass[sigconf]{acmart}

\AtBeginDocument{%
  \providecommand\BibTeX{{%
    \normalfont B\kern-0.5em{\scshape i\kern-0.25em b}\kern-0.8em\TeX}}}

\copyrightyear{2022}
\acmYear{2022}
\setcopyright{acmcopyright}\acmConference[SIGMOD '22]{Proceedings of the 2022 International Conference on Management of Data}{June 12--17, 2022}{Philadelphia, PA, USA}
\acmBooktitle{Proceedings of the 2022 International Conference on Management of Data (SIGMOD '22), June 12--17, 2022, Philadelphia, PA, USA}
\acmPrice{15.00}
\acmDOI{10.1145/3514221.3526127}
\acmISBN{978-1-4503-9249-5/22/06}

\acmSubmissionID{00000}

\usepackage{amsmath}
\usepackage{amsthm}
\usepackage{amssymb}
\usepackage{mathtools}
\usepackage{bm}
\usepackage{bbm}
\usepackage{booktabs}
\usepackage{xspace}
\usepackage{color,soul}
\usepackage{multirow}
\usepackage{multicol}
\usepackage{caption}
\usepackage{enumitem}

\usepackage{microtype}
\usepackage{graphicx}
\usepackage{subfigure}
\usepackage{booktabs}
\usepackage{hyperref}
\usepackage{nicefrac}
\usepackage[normalem]{ulem}
\def\halfcheckmark{{\checkmark}\textsuperscript{{\kern-0.55em\tiny$\times$}}}

\newcommand{\specialcell}[2][c]{%
	\begin{tabular}[#1]{@{}c@{}}#2\end{tabular}}
\newcommand{\mytextcircled}[1]{\textcircled{\raisebox{-0.8pt}{#1}}}
\newcommand{\mydot}{\diamond}
\newcommand{\otherdot}{{\bar{\diamond}}}

\newcommand{\host}{\textit{Party A}\xspace}
\newcommand{\guest}{\textit{Party B}\xspace}
\newcommand{\hostul}{\ul{\textit{Party A}}\xspace}
\newcommand{\guestul}{\ul{\textit{Party B}}\xspace}
\newcommand{\party}{\textit{Party}\xspace}
\newcommand{\enc}[1]{\llbracket {#1} \rrbracket}
\newcommand{\enca}[1]{\enc{#1}_{A}}
\newcommand{\encb}[1]{\enc{#1}_{B}}
\newcommand{\pka}{pk_{A}}
\newcommand{\pkb}{pk_{B}}
\newcommand{\pkdot}{pk_{\mydot}}
\newcommand{\pkotherdot}{pk_{\otherdot}}
\newcommand{\ska}{sk_{A}}
\newcommand{\skb}{sk_{B}}

\newcommand{\skotherdot}{sk_{\otherdot}}
\renewcommand{\ss}[1]{\langle {#1} \rangle}

\newcommand{\xa}{X_A}
\newcommand{\xb}{X_B}
\newcommand{\xdot}{X_\mydot}
\newcommand{\xodot}{X_\otherdot}
\newcommand{\tw}{W}
\newcommand{\twa}{\tw_A}
\newcommand{\twb}{\tw_B}
\newcommand{\twdot}{\tw_\mydot}
\newcommand{\ua}{U_A}
\newcommand{\ub}{U_B}
\newcommand{\udot}{U_\mydot}

\newcommand{\va}{V_A}
\newcommand{\vb}{V_B}
\newcommand{\vdot}{V_\mydot}
\newcommand{\vodot}{V_\otherdot}
\newcommand{\qa}{Q_A}
\newcommand{\qb}{Q_B}
\newcommand{\qdot}{Q_\mydot}

\newcommand{\sa}{S_A}
\renewcommand{\sb}{S_B}
\newcommand{\sdot}{S_\mydot}

\newcommand{\ta}{T_A}
\newcommand{\tb}{T_B}
\newcommand{\tdot}{T_\mydot}

\newcommand{\lk}{\textup{lkup}}
\newcommand{\dlk}{\textup{lkup\_bw}}
\newcommand{\ea}{E_A}
\newcommand{\eb}{E_B}
\newcommand{\edot}{E_\mydot}

\definecolor{myred}{rgb}{1.0,0.7,0.8}
\definecolor{mygreen}{RGB}{0,166,0}
\definecolor{lightgreen}{rgb}{0.56, 0.93, 0.56}
\definecolor{myorange}{RGB}{252,107,4}
\definecolor{darkgreen}{RGB}{0,153,102}
\definecolor{lightblue}{rgb}{0.53, 0.81, 0.92}
\definecolor{lightgray}{gray}{0.9}

\newcommand{\eps}{\varepsilon}

\newtheorem*{assumption*}{\assumptionnumber}
\providecommand{\assumptionnumber}{}


\newtheorem*{adversary*}{\adversarynumber}
\providecommand{\adversarynumber}{}


\usepackage{tikz}
\usetikzlibrary{trees}
\usepackage{float}
\usepackage{subfigure}
\usepackage[ruled,noend,linesnumbered,vlined]{algorithm2e}
\SetKwInput{kwFeatures}{Features}
\SetKwInput{kwWeights}{Weights}
\SetKwInput{kwEmbeddings}{Embeddings}
\SetKwInput{kwParams}{Hyper-parameters}
\SetKwInput{kwDenote}{Denote}

\newcommand{\alg}{\textsc{BlindFL}\xspace}

\newcommand{\mysubsubsection}[1]{{\textbf{\textit{#1}.\xspace}}}
\newcommand{\myparagraph}[1]{{\textbf{#1}\xspace}}

\usepackage{tikz}
\usepackage[framemethod=tikz]{mdframed}
\usetikzlibrary{shadows}
\newmdenv[shadow=false,font=\small,backgroundcolor=white,leftmargin=0pt,rightmargin=0pt]{shadedbox}

\begin{document}
\fancyhead{}

\title{{\alg}: Vertical Federated Machine Learning without \\Peeking into Your Data}


\author{Fangcheng Fu}
\authornote{School of Computer Science \& Key Lab of High Confidence Software Technologies (MOE), Peking University}
\affiliation{
\institution{Peking University}
}
\email{ccchengff@pku.edu.cn}

\author{Huanran Xue}
\affiliation{
\institution{Tencent Inc.}
}
\email{huanranxue@tencent.com}

\author{Yong Cheng}
\affiliation{
\institution{Tencent Inc.}
}
\email{peterycheng@tencent.com}

\author{Yangyu Tao}
\affiliation{
\institution{Tencent Inc.}
}
\email{brucetao@tencent.com}

\author{Bin Cui}
\authornotemark[1]
\authornote{Institute of Computational Social Science, Peking University (Qingdao), China}
\affiliation{
\institution{Peking University}
}
\email{bin.cui@pku.edu.cn}


\begin{abstract}
Due to the rising concerns on privacy protection, 
how to build machine learning (ML) models 
over different data sources with security guarantees 
is gaining more popularity. 
Vertical federated learning (VFL) describes such a case 
where ML models are built upon the private data of 
different participated parties that own disjoint features 
for the same set of instances, 
which fits many real-world collaborative tasks. 
Nevertheless, we find that existing solutions for VFL 
either support limited kinds of input features 
or suffer from potential data leakage 
during the federated execution. 
To this end, this paper aims to investigate 
both the functionality and security 
of ML modes in the VFL scenario. 

To be specific, we introduce \alg, 
a novel framework for VFL training and inference. 
First, to address the functionality of VFL models, 
we propose the federated source layers 
to unite the data from different parties. 
Various kinds of features 
can be supported efficiently 
by the federated source layers, 
including dense, sparse, numerical, 
and categorical features. 
Second, we carefully analyze the security 
during the federated execution 
and formalize the privacy requirements. 
Based on the analysis, 
we devise secure and accurate algorithm protocols, 
and further prove the security guarantees 
under the ideal-real simulation paradigm. 
Extensive experiments show that 
\alg supports diverse datasets and models efficiently 
whilst achieves robust privacy guarantees. 
\end{abstract}



\maketitle

\input{sections/intro}

\input{sections/pre}

\input{sections/anatomy}

\input{sections/overview}

\input{sections/matmul}

\input{sections/embed}

\input{sections/exper}

\input{sections/related}

\input{sections/conc}

\begin{acks}
This work is supported by National Natural Science Foundation of China (NSFC) (No. 61832001), PKU-Tencent joint research Lab, and Beijing Academy of Artificial Intelligence (BAAI). Bin Cui is the corresponding author.
\end{acks}

\balance

\nocite{het,partial_reduce,openbox,pasca,snapshot_boosting,real_time_rec,cl4rec,tinyscript,sketchml,vero,dimboost,dts_moe}
\bibliographystyle{ACM-Reference-Format}
\bibliography{reference}

\newif\ifappendix
\appendixtrue 
\ifappendix
\appendix
\input{sections/appendix.tex}

\fi

\end{document}

%% file: sections/intro.tex
\section{Introduction}
\label{sec:intro}

\begin{figure}[!t]
\centering
\includegraphics[width=3.4in]{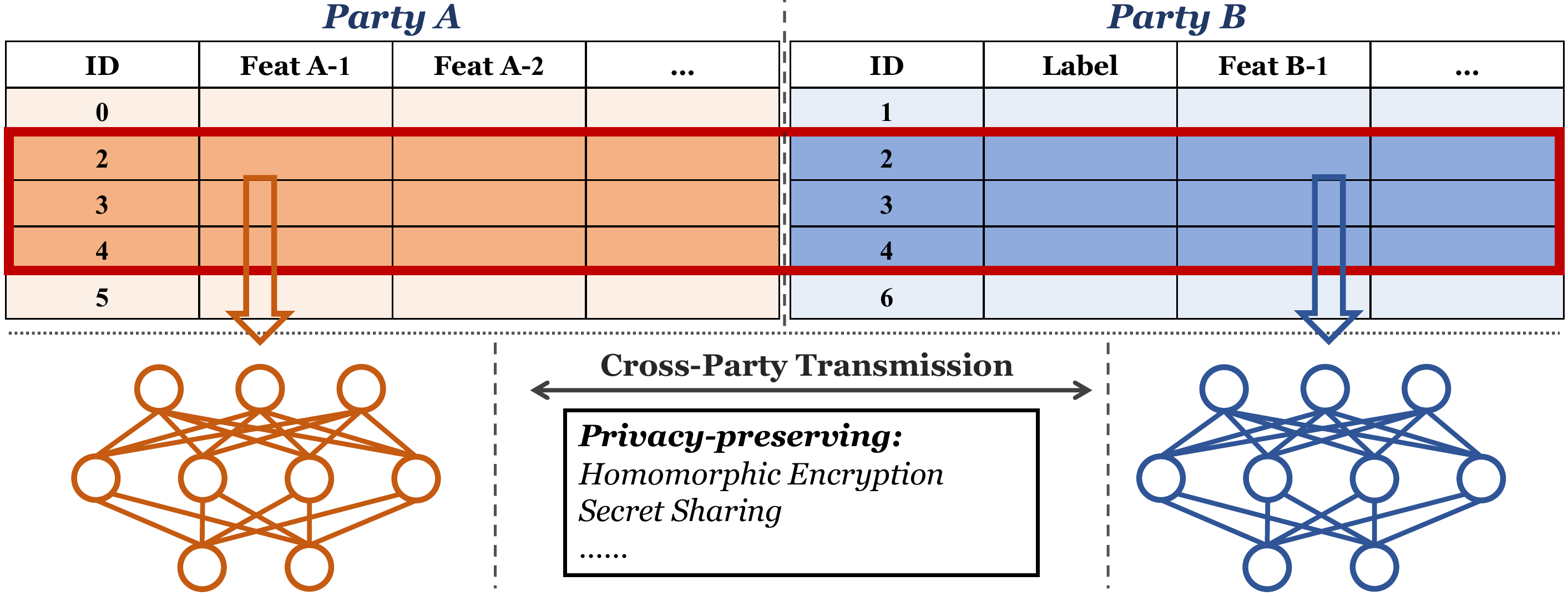}
\caption{{
An illustration of vertical federated learning (VFL).
Two parties own disjoint features
but have overlapping instances. 
Privacy-preserving techniques are utilized 
to protect the data in each party.
}}
\label{fig:vertical_fl}
\end{figure}

\mysubsubsection{Background and Motivation}
In recent years, 
following the explosive surge of data volume and 
the remarkable success of machine learning (ML) 
in the whole world, 
more and more enterprises are thirsty to 
collect tremendous user data, such as 
media data, text messages, and daily locations, 
to build better ML models. 
Meanwhile, it leads to the increasingly 
notorious abuse of personal data 
or even illegal leakage of individual privacy. 
Thus, the society has raised 
a growing attention to 
the protection of data privacy and 
the supervision on the potential risks to 
data leakage. 
Enormous lawful regulations 
have been established to 
protect the individual privacy
~\cite{gdpr,ccpa,cdpa}. 
Consequently, many enterprises are now 
restricted from collecting a great deal of 
data for ML tasks.

Owing to such a dilemma of ``data shortage'', 
researchers and data scientists are interested in 
building ML models with the data of 
different parties (typically, enterprises or organizations) 
on the basis of zero data leakage for each individual party. 
Specifically, 
a new paradigm called Federated Learning (FL) 
~\cite{fl_concepts,fl_challenge,konevcny2016federated_opt,konevcny2016federated_learn,vfl_ea,McMahan2017_fl}
conveys a possibility to train ML models 
over multiple data sources that are physically distributed 
with privacy preservation. 
In this work, we consider the vertical FL (VFL) scenario. 
As illustrated in Figure~\ref{fig:vertical_fl}, 
two participated parties own 
disjoint features (a.k.a. attributes) 
but overlap on some instances (a.k.a. samples). 
The overlapping instances together 
form a virtually joint dataset, 
which is vertically partitioned into two parties. 
\guest further holds the ground truth labels. 
VFL inputs the virtually joint dataset 
and outputs a federated model that is 
trained collaboratively.

\begin{figure*}[!t]
\centering
\includegraphics[width=\linewidth]{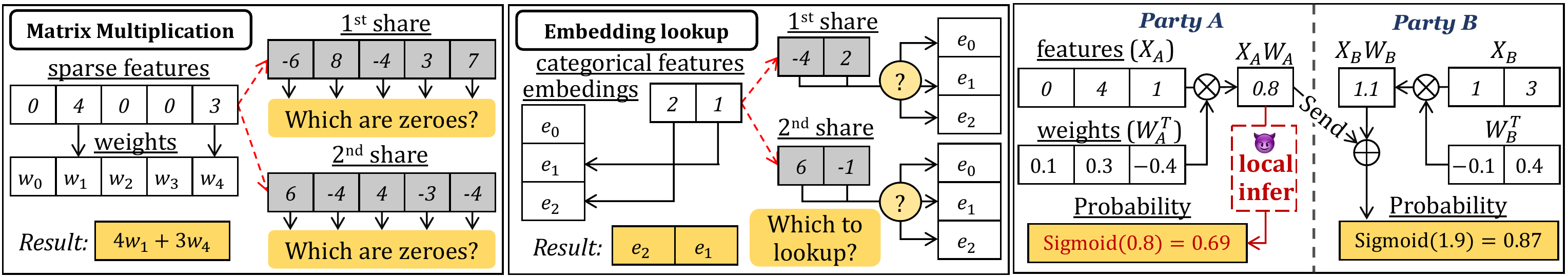}
\caption{{
Limitations of existing paradigms. 
Left-most and middle: 
Data outsourcing is not suitable for 
sparse and/or categorical features 
(illustrated by two examples that secretly share features). 
Right-most: 
A simple example of logistic regression 
following the split learning paradigm. 
\host can infer the labels accurately 
since the bottom model (i.e., $\twa$) is accessible. 
}}
\label{fig:limits}
\end{figure*}

Our industrial partner runs popular social apps 
and is able to gather rich user data 
and precise user profiling. 
Many collaborators wish to improve the ML ability 
of their tasks 
with the help of the data of our industrial partner, 
such as a Fintech company that 
hopes to build a more powerful risk model,  
or an E-commerce company that 
wishes to make more accurate recommendation{\footnote{
Formally speaking, there can be 
more than one \host's. 
However, most cross-enterprise collaboration 
follows the two-party setting. 
Furthermore, our work 
can be easily generalized to more \host's. 
Thus, we only describe the case of 
only one \host for simplicity, 
whilst discuss the multi-party setting 
in our appendix.}}. 
To this end, VFL is a good fit for 
such kind of cross-enterprise collaboration.

To be formal, the goal of VFL 
is to unite the features of \host and \guest, 
which are denoted as $\xa, \xb$, respectively, 
to learn a federated model that 
fits the target labels $y$ in \guest, 
without any privacy leakage of data in both parties. 
In this work, we focus on tabular data 
since it is rare for an image or a sentence 
to be split and distributed into different parties. 
Most importantly, the federated model should achieve 
comparable performance (e.g., accuracy, loss) as 
the non-federated model trained on collocated datasets.

\mysubsubsection{Challenges} 
We find that the existing VFL solutions can be 
categorized into two different lines according to 
how the private features are processed. 
Nevertheless, the two existing paradigms 
have complementary strengths and drawbacks.

The first line of works 
~\cite{secureml,aby3,cryptonets,cryptodl} 
leverages secure multi-party computation 
(MPC)~\cite{smpc} techniques, such as 
homomorphic encryption (HE) and secret sharing (SS), 
to achieve intact privacy guarantees. 
Typically, they utilize the data outsourcing technique, 
which is widely used in database services
~\cite{dbaas,secure_data_outsource}, 
and outsource the datasets to non-colluding servers 
for ML training or inference. 
To maintain privacy, 
all feature values are turned into 
HE or SS variables when outsourcing 
so that the servers cannot know the original values. 
However, such an approach does not fit 
many real-world datasets. 
We provide two examples 
that secretly share features onto two servers 
in Figure~\ref{fig:limits}. 
First, for high-dimensional and sparse datasets, 
the outsourced features become fully dense 
since the non-zero feature indexes 
are also private. 
It prevents us from sparsifying the computation, 
leading to a performance hazard when the sparsity is high. 
Second, for categorical features, 
the embedding lookup operation requires 
knowing the exact categorical values. 
However, performing lookup operations 
on the outsourced values is invalid. 
As a result, although these MPC-based methods 
have promising security guarantees, 
the data outsourcing nature is not suitable for 
sparse and/or categorical features.

Another line of works follows the split learning paradigm
~\cite{split_learning,split_fed,vfl_lr,interactive,vfl_label_protect,fdml,async_vfl}, 
which does not outsource the original datasets 
so that sparse and categorical features 
can be handled well. 
Typically, each party maintains a bottom model in plaintext 
that extracts the forward activations 
(a.k.a. hidden representations) 
using its own private features. 
The activations of all parties will be exchanged 
and fed into a top model to make predictions. 
(See Section~\ref{sec:pre} for more details.)
Nevertheless, the values generated in the bottom models 
are released in plaintext and 
would cause data leakage. 
For instance, Figure~\ref{fig:limits} illustrates 
a simple example for logistic regression. 
Although \host cannot get access to \guest's data, 
it can still infer the labels accurately 
by analyzing $\xa\twa$, 
because the bottom model (i.e., $\twa$) is managed by \host. 
Undoubtedly, such data leakage 
is forbidden and even illegal in 
real-world applications. 
As we will analyze in Section~\ref{sec:anatomy}, 
although these works are more flexible 
to support sparse and/or categorical features, 
such a design of local bottom models 
cannot provide provable security guarantees 
like the MPC-based methods, 
and thus there exist potential safety hazards.

\begin{table}[!t]
\small
\centering
\caption{\small{Comparison of different works in terms of (i) whether they are suitable for these types of features; (ii) whether they have provable security guarantees under the ideal-real simulation paradigm.}}
\begin{tabular}{|c|c|c|c|c|}
\hline
\multirow{3}*{\specialcell{\textbf{Paradigm \& How}\\\textbf{Features are Processed}}} & 
\multicolumn{3}{c|}{\textbf{Supported Features}} & 
\multirow{3}*{\specialcell{\textbf{Security}\\\textbf{Guarantees}}} \\
\cline{2-4}
& \multicolumn{2}{c|}{\scriptsize{Numerical}} 
& \multirow{2}*{\scriptsize{Categorical}}
& \\
\cline{2-3}
& \scriptsize{Dense} & \scriptsize{Sparse} & & \\
\hline
\hline
\specialcell{MPC-based\\(Data Outsourcing)}
& \checkmark & & & \checkmark \\
\hline
\specialcell{Split Learning\\(Local Bottom Model)}
& \checkmark & \checkmark & \checkmark & \\
\hline
\specialcell{\alg (this work)\\(Federated Source Layer)}
& \checkmark & \checkmark & \checkmark & \checkmark \\
\hline
\end{tabular}
\label{tb:summary} 
\end{table}

\mysubsubsection{Summary of Contributions}
As shown in Table~\ref{tb:summary}, 
the aforementioned paradigms 
either support limited feature types 
or fail to convey promising security guarantees. 
Motivated by this, 
we develop a novel VFL framework, 
namely \alg (pronounced as ``blindfold''), 
to address these challenges.
The major contributions 
of our work are summarized as follows. 

\myparagraph{Proposal of \alg.}
We propose \alg, a brand new framework for 
VFL training and inference. 
\alg keeps the private datasets inside each party 
without outsourcing 
and unites the features by an important component called 
``federated source layer''. 
By doing so, \alg can support various kinds of 
input features, including dense, sparse, 
numerical, and categorical features, 
whilst achieves security guarantees in the meantime. 
\alg has been deployed in many productive applications 
of our industrial partner. 

\myparagraph{Analysis of Privacy Requirements.}
We anatomize the privacy requirements 
of VFL training and inference thoroughly. 
To be specific, 
we present a comprehensive analysis of 
the informativeness of 
all kinds of values generated in the learning process, 
including forward activations, backward derivatives, 
model weights, and model gradients, 
i.e., how it would cause leakage 
once they are obtained by a party. 
Upon the analysis, we formulate several privacy requirements 
--- the detailed contents 
that each party must not get access to 
during the execution. 
These privacy requirements 
shed light on how to judge whether a federated algorithm 
is secure or not and 
provide a template to design new algorithm protocols. 

\myparagraph{Design of Algorithm Protocols.}
Based on the privacy requirements, 
we devise the federated source layers, 
the basic building blocks 
to unite the features from different data sources. 
The federated source layer 
leverages the HE and SS techniques 
to accomplish the aforementioned privacy requirements 
throughout the algorithm protocols. 
Two kinds of federated source layers, namely 
\texttt{MatMul} and \texttt{Embed-MatMul}, 
are designed for numerical features and 
categorical features, respectively. 
We further prove that our algorithm protocols 
are secure in the presence of semi-honest adversaries 
that can corrupt up to one party 
under the ideal-real paradigm.
With these two kinds of source layers, 
\alg is able to support various kinds of features 
and build diverse VFL models, including 
generalized linear models (GLMs) and 
neural networks (NNs). 

\myparagraph{Experimental Evaluation.}
Comprehensive experiments are conducted to evaluate 
the effectiveness of \alg. 
First, empirical results show that 
\alg is more robust against the semi-honest adversaries 
and protects the data privacy well. 
Second, our work outperforms the existing works 
over 50$\times$ in terms of running speed 
and supports much larger scale of datasets. 
Third, extensive experiments on 
a wide range of datasets and models 
prove that 
\alg achieves comparable model performance 
as non-federated learning on collocated datasets, 
verifying its lossless property.

%% file: sections/pre.tex
\section{Preliminaries}
\label{sec:pre}
In this section, we briefly introduce the 
preliminary literature related to our work. 
For the sake of simplicity, in the rest of the paper, 
we use the symbol ``$\mydot$'' to represent 
an arbitrary party 
and ``$\otherdot$'' to represent the other party, 
respectively.

\subsection{Common ML Ops for Input Features}
Given a dataset $\langle X, y \rangle$ 
and a loss function $f$, 
where $X$ is the features and 
$y$ is the labels, 
our goal is to learn a model $\theta$ that 
predicts $\hat{y}$ for $X$ 
and minimizes the loss $f(y, \hat{y})$. 
The most prevailing way to solve 
such a supervised ML problem 
is to use the first-order gradient optimization, 
typically, the mini-batch stochastic gradient descent 
(SGD) and its variants. 
In each iteration, a mini-batch of instances 
$\langle X^{(B)}, y^{(B)} \rangle$ 
is sampled to calculate the model gradients 
$\nabla\theta = 
\nicefrac{\partial f(y^{(B)}, \hat{y}^{(B)})}{\partial\theta}
$ and the model weights are updated via 
$\theta = \theta - \eta\nabla\theta$, 
where $\eta$ is the learning rate (a.k.a. step size). 
In the rest of the paper, we omit the superscript ${(B)}$ 
for simplicity.

\textbf{Matrix Multiplication.\xspace}
Matrix multiplication is one of the most 
common operations in ML. 
Given the \textit{model weights} 
$\tw \in \mathbb{R}^{IN \times OUT}$, 
where $IN, OUT$ are the input and output dimensionalities, 
it computes $Z = X \tw$ in the forward propagation, 
which are also known as the \textit{forward activations}. 
The subsequent modules will take as input $Z$ 
to perform their forward propagation routines. 
During the backward propagation, 
the \textit{backward derivatives} $\nabla Z = 
\nicefrac{\partial f}{\partial Z}$ are propagated 
from the subsequent modules and 
the \textit{model gradients} are computed as 
$\nabla \tw = X^T \nabla Z$ according to the chain rule. 

\textbf{Embedding Lookup.\xspace}
For categorical features, applying 
the matrix multiplication is not a common choice in ML 
since the ordering of categorical values 
should not matter. 
In contrast, an embedding table is usually learned 
for the categorical inputs{\footnote{
Feature engineering techniques such as 
one-hot encoding or frequency encoding 
can also be applied to the categorical features, 
however, we do not discuss them in this work 
since they are orthogonal.}}.
In this work, the embedding table is denoted as $Q$. 
During the forward period, 
an embedding lookup operation $E = \lk(Q, X)$ 
queries the embedding entries 
given the categorical indices. 
The backward propagation computes the model gradients 
as $\nabla Q = \dlk(\nabla E, X)$. 
Embedding lookup is usually followed by 
a matrix multiplication, 
which outputs $Z = E \tw$ 
given the model weights $\tw$.

\subsection{Privacy-Preserving Techniques}
Security is a widely studied area to 
protect private data from any adversaries. 
In this work, we adopt two well-known
privacy-preserving techniques 
--- homomorphic encryption and secret sharing ---
to derive our federated algorithms.

\textbf{Homomorphic Encryption.\xspace}
Homomorphic encryption (HE)~\cite{homo,survey_homo} 
describes the cryptographic methods 
that allow arithmetic computation in the space of ciphers. 
There are different kinds of HE, 
such as fully HE, somewhat HE, additive HE, 
and multiplicative HE
~\cite{gentry2009fully_homo,somewhat_he,paillier,ckks_homo,rsa}. 

In this work, we focus on additive HE. 
For instance, Paillier cryptosystem~\cite{paillier} 
is a well-known additive HE method 
and has been used in many FL algorithms
~\cite{vfboost,secureboost,vfl_lr,private_vlr,privacy_vfl_tree}. 
Paillier cryptosystem initializes with 
a key pair $\langle pk, sk \rangle$. 
The public key $pk$ is utilized in encryption 
and can be made public to the other party, 
whilst the secret key (a.k.a. private key) $sk$ 
is for decryption and must be kept secret. 
Given values $u, v$, 
Paillier cryptosystem 
supports the following types of operations: 
\begin{itemize}[leftmargin=*,label=$\triangleright$]
\item
Encryption: 
$\textup{\ul{Enc}}(v, pk) = \enc{v}$; 
     
\item
Decryption: 
$\textup{\ul{Dec}}(\enc{v}, sk) = v$; 
     
\item
Homomorphic addition: 
$\enc{u} + \enc{v} = \enc{u + v}$; 
     
\item
Scalar addition: 
$\enc{u} + v = \enc{u} + \textup{\ul{Enc}}(v, pk) = \enc{u + v}$; 

\item
Scalar multiplication: 
$u\enc{v} = \enc{uv}$.
\end{itemize}

In the rest of the paper, we assume 
both parties have generated their own key pairs 
and exchanged the public keys on initialization. 
Thus, we omit $pk$ or $sk$ 
and denote $\enc{v}_\mydot$ as a cipher 
corresponding to the $sk$ of \party $\mydot$.

\textbf{Secret Sharing.\xspace}
As another powerful privacy-preserving technique, 
secret sharing (SS)~\cite{ss_shamir,ss_arithmetic,aby,aby3,unbalance_sharing_sm4} 
breaks a value into pieces of sharing 
and distributes them to 
different parties so that none of the parties 
knows the exact value. 
For instance, 
if \party $\mydot$ wishes to secretly share a value $v$, 
it randomly generates its own piece of sharing $v_\mydot$ 
and sends the other piece $v_\otherdot = v - v_\mydot$ 
to \party $\otherdot$. 
Whenever a party wishes to restore an SS variable, 
it receives the other piece of sharing 
and simply adds the two pieces together.

In this work, we focus on two-party additive SS. 
Given two SS variables 
$\langle v_A, v_B \rangle, \langle u_A, u_B \rangle$, 
the addition arithmetic can be executed locally as 
$\langle u_A + v_A, u_B + v_B \rangle$. 
A common way to multiply two SS variables 
is based on the Beaver triplets. 
However, each triplet can only be used once 
and it is time-consuming to generate. 
We refer interested readers to~\cite{beaver,beaver_gen} 
for more details.

\begin{algorithm}[!t]
\caption{{The procedure to transform an HE variable $\enc{v}$ 
into an SS variable $\langle \phi, v - \phi \rangle$.
$v$ is a scalar or a tensor.}}
\label{alg:he2ss}
\DontPrintSemicolon
\SetKwFunction{HEtoSS}{HE2SS}
\SetKwProg{Fn}{Function}{:}{}

\Fn{\HEtoSS{$\enc{v}_\otherdot=None$}}
{
	\uIf (\tcp*[h]{\party $\otherdot$ that owns $\skotherdot$}) 
	{$\enc{v}_\otherdot$ is None}
	{
		Receive $\enc{v - \phi}_\otherdot$ and decrypt;
		\Return $v - \phi$
	}
	\uElse (\tcp*[h]{\party $\mydot$ that does not own $\skotherdot$})
	{
		Randomly generate $\phi$ with the same shape of $\enc{v}_\otherdot$\; 
		Send $\enc{v - \phi}_\otherdot = \enc{v}_\otherdot - \phi$;
		\Return $\phi$
	}
}
\end{algorithm}

It is worthy to note that HE and SS variables 
can be transformed easily. 
For instance, we can transform an HE variable to SS variable 
via Algorithm~\ref{alg:he2ss}. 
Moreover, the operations on HE or SS variables 
can be easily vectorized. 
For instance, 
the matrix multiplication $\enc{Z} = X\enc{Y}$ 
is achieved by letting 
$\enc{Z_{ij}} = \sum_k X_{ik} \enc{Y_{kj}}$.

\textbf{Security model.\xspace}
Like many previous works
~\cite{secureml,securenn1,securenn2,vfboost,privacy_vfl_tree}, 
we consider the semi-honest 
(a.k.a. honest-but-curious) security model, 
where all parties honestly execute the protocols,
whilst the curious parties try to analyze the others' data 
through any information obtained during the protocols. 
We assume that a polynomial-time adversary 
can corrupt one of the two parties during the execution. 
When analyzing our protocols, we follow 
the ideal-real paradigm, which is defined as follows. 
\begin{definition}
(\cite{foundations_of_crypto,how_to_simulate})
Let $\Pi$ be a real-world protocol and 
$\mathcal{F}$ be an ideal functionality. 
We say $\Pi$ securely realizes $\mathcal{F}$ 
if for every adversary $\mathcal{A}$ 
attacking the real interaction, 
there exist a probabilistic polynomial-time simulator 
$\mathcal{S}$ attacking the ideal interaction, 
such that any environment on any input 
cannot tell apart 
the real interaction from the ideal interaction, 
except with negligible probability 
(in the security parameter $\kappa$). 
\end{definition}

\subsection{ML over Different Data Sources}
Under the VFL setting, training data come from 
different sources. 
The most challenging problem is  
how to build ML models over different data sources 
to achieve comparable model performance as 
non-federated learning on collocated data, 
whist guarantee the data privacy of 
each participated party meanwhile. 
A series of works are developed to tackle this problem. 
We divide them into two categories 
according to how they process the features. 

\textbf{MPC-based solutions (data outsourcing).}
The first kind of works outsources the data of all parties 
to one or more non-colluding servers 
for ML training or inference. 
Before data outsourcing, they transform the feature values 
into HE or SS variables and leverage the arithmetic properties 
to perform ML ops. 
For instance, SecureML~\cite{secureml} 
secretly shares the features and models 
onto two servers. 
Then the matrix multiplication is performed as 
$\ss{X}\ss{W}$. 
Similarly, there are also works that 
encrypt and send the features to a single server, 
and utilize the arithmetic properties of HE 
to carry out the computation~\cite{cryptodl,cryptonets}.
With the help of the HE and SS techniques, 
none of the servers can reveal the original values, 
hence these works usually achieve 
provable security guarantees 
under the ideal-real paradigm. 
To apply these works under the VFL setting, 
we can let each party work as one server 
and outsource the features among the parties 
for ML computation.

\begin{figure}[!t]
\centering
\includegraphics[width=3in]{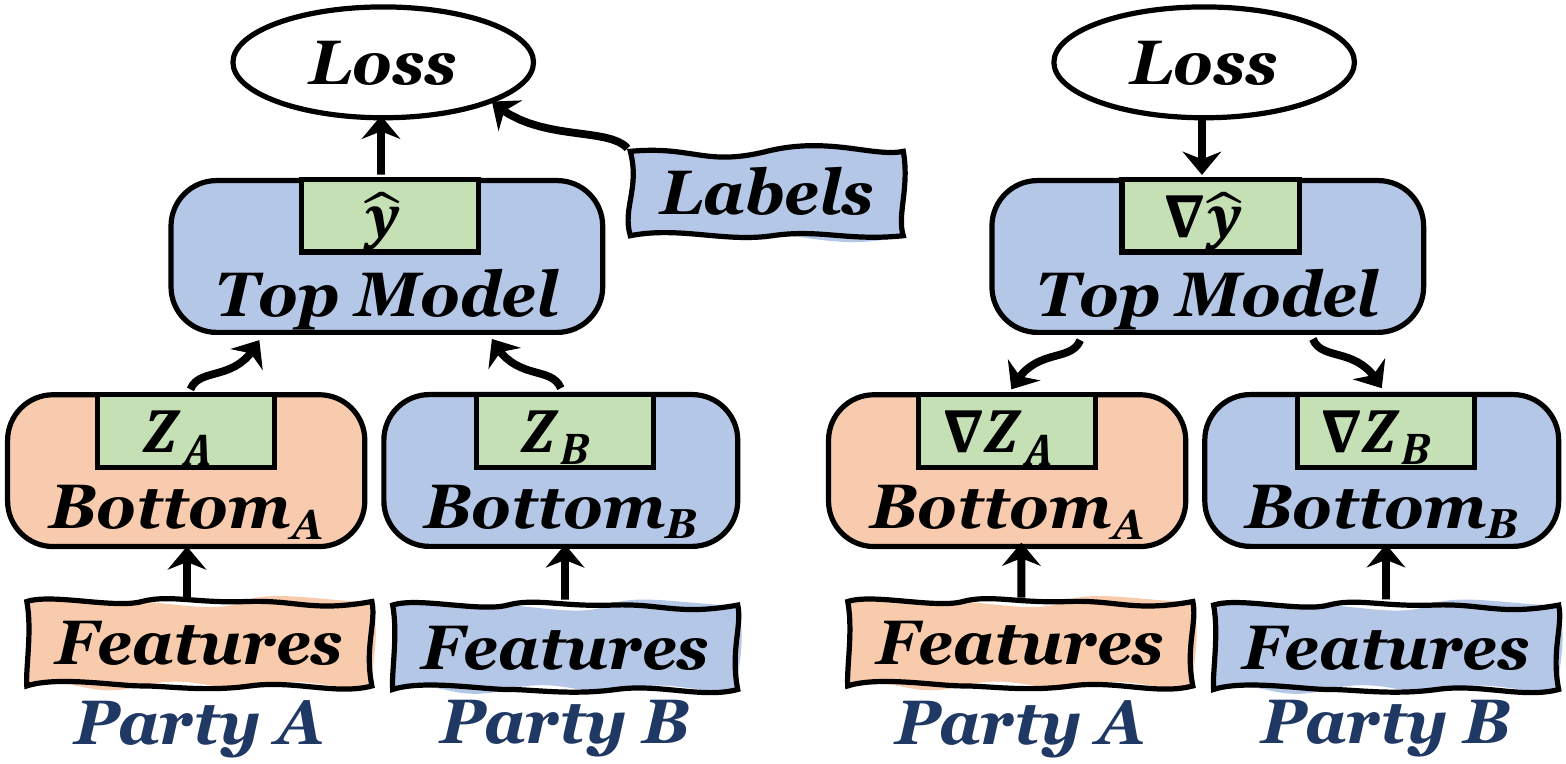}
\caption{{The bottom-to-top architecture of split learning.}}
\label{fig:split_learning}
\end{figure}

\begin{figure*}[!t]
\begin{minipage}{.45\textwidth}
    \centering
    \includegraphics[width=3in]{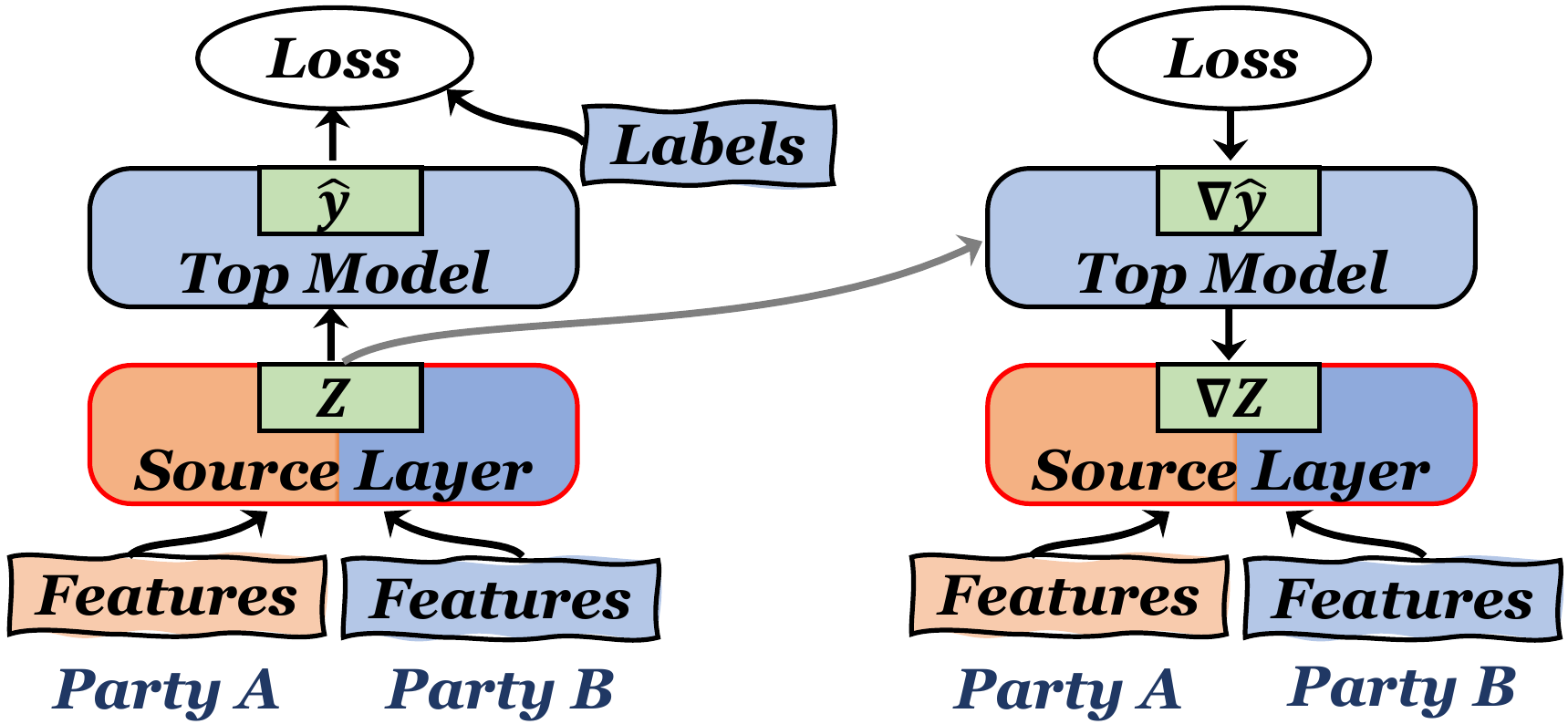}
    \captionof{figure}{{Overview of the forward (left) and backward (right) propagation of \alg.}}
    \label{fig:overview}
\end{minipage}
\begin{minipage}{.01\textwidth}
$ $
\end{minipage}
\begin{minipage}{.5\textwidth}
    \centering
    \includegraphics[width=3in]{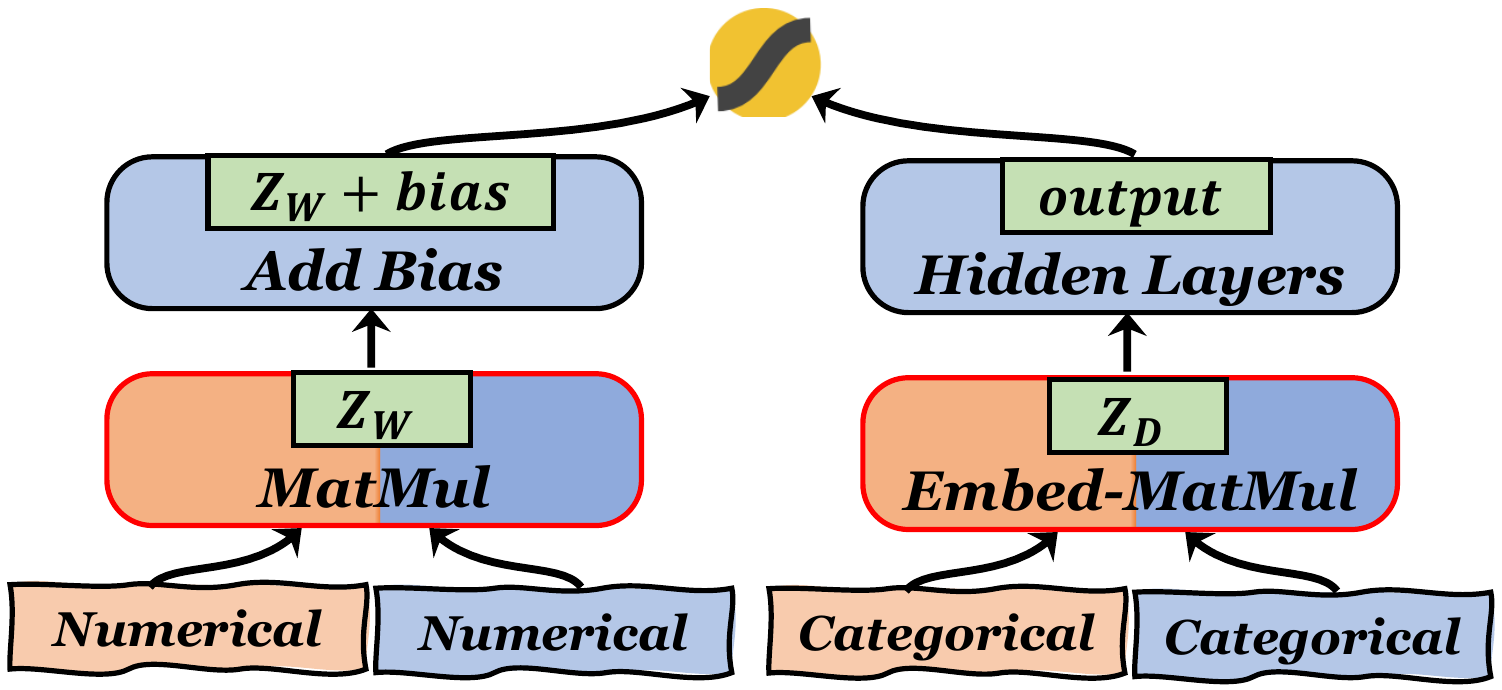}
    \captionof{figure}{{Example of the wide and deep (WDL) model tailored for \alg. There are two source layers --- \texttt{MatMul} handles the sparse, numerical features (the wide part) and \texttt{Embed-MatMul} handles the categorical ones (the deep part).}}
    \label{fig:wdl_overview}
\end{minipage}
\end{figure*}

\textbf{Split learning (local bottom models).}
The second category follows 
the split learning paradigm~\cite{split_learning,split_fed}. 
As shown in Figure~\ref{fig:split_learning}, 
each party holds its own data rather than outsourcing. 
Each party is associated with a bottom model 
that works as the feature extractor. 
\guest further manages a top model that 
makes the final predictions. 
The forward activations $Z_A$ and 
backward derivatives $\nabla Z_A$ are exchanged 
between the parties. 
With such a design of bottom models, 
split learning is very flexible to 
support various kinds of features. 
For instance, to support sparse matrix multiplication, 
the bottom model of \party $\mydot$ computes 
$\xdot\twdot$ using sparsified computation, 
where $\twdot$ is the model weights, 
and the top model of \guest aggregates $Z=\xa\twa+\xb\twb$. 
To support categorical features, 
each party manages an embedding table in the bottom model 
and performs the lookup operation locally. 
Thus, several VFL algorithms are proposed in such a pattern
~\cite{vfl_lr,interactive,vfl_label_protect,fdml,async_vfl}.

%% file: sections/anatomy.tex
\section{Anatomy of Existing Paradigms}
\label{sec:anatomy}

In this section, we anatomize the existing paradigms 
from the perspective of supported features and 
security guarantees, respectively.

\mysubsubsection{The MPC-based solutions}
We first consider the MPC-based solutions, 
which outsource the features 
using privacy-preserving techniques 
to carry out the ML computation. 

\textbf{Supported features}. 
As introduced in Section~\ref{sec:intro}, 
data outsourcing is not suitable for 
sparse and/or categorical features. 
On the one hand, sparse datasets will become fully dense, 
since the outsourced variables should not reveal 
whether the original values are zeroes or not. 
On the other hand, after the categorical values 
are transformed into HE or SS variables, 
the embedding lookup operation cannot be performed. 
However, many ML tasks adopt feature engineering techniques 
such as hashing, quantile binning, and Cartesian product, 
which make many real-world datasets 
high-dimensional and sparse. 
In addition, categorical features are also widely used 
in practical applications to learn latent representation 
for better model performance. 
Therefore, these outsourcing-based solutions 
are not suitable for many real-world datasets. 

\textbf{Security guarantees}. 
Despite the aforementioned limitations, 
the MPC-based solutions usually achieve 
very promising security guarantees. 
For instance, SecureML~\cite{secureml} 
securely realizes the ideal functionality 
of ML training in the presence of semi-honest adversaries. 
As a result, methods in this line 
have a strong privacy preservation ability. 

\mysubsubsection{Split learning}
Next, we consider the methods in the split learning paradigm, 
where each party is able to 
process the features locally via a bottom model.

\textbf{Supported features}. 
Compared with the outsourcing-based methods, 
split learning is more suitable for 
sparse and/or categorical features 
since all data are kept local inside each party. 
During the ML tasks, each party can easily 
identify whether a feature value is zero 
or perform the lookup operation given a categorical value.

\textbf{Security guarantees}. 
Since the features are visible to the owner party, 
it is an intuitive idea to 
let each party process its own features locally 
and only share the intermediate results between parties. 
However, although the raw data are not exposed directly, 
the intermediate results are also informative 
and would cause severe data leakage. 
To be specific, we identify two kinds of leakage: 
\begin{itemize}[leftmargin=*,label=$\triangleright$]
\item Label leakage. 
On the one hand, since \host owns the bottom model, 
it is able to compute $Z_A$ individually, 
which would leak the labels. 
Take Figure~\ref{fig:limits} as an example. 
\host can steal the labels by analyzing $Z_A = \xa\twa$ 
with high confidence. 
Many existing works are vulnerable to this problem
~\cite{vfl_lr,async_vfl,private_vlr}. 
On the other hand, as analyzed by~\citet{vfl_label_protect} 
and verified by our experiments, 
the backward derivatives $\nabla Z_A$ can reveal 
almost all training labels. 
The essential idea is that logistic loss produces 
opposite directions for different labels, 
so the directions of derivatives 
reflect the label information inevitably.
\item \host's feature leakage. 
Since the forward activations $Z_A$ 
are originated from $\xa$ and independent from $\xb$, 
\guest can analyze a certain level of feature similarities 
of $\xa$. 
In other words, if the features of two instances 
{\footnotesize $\xa^{(i)}, \xa^{(j)}$} 
are very similar, the corresponding activations 
{\footnotesize $Z_A^{(i)}, Z_B^{(j)}$} 
would also be very close. 
\end{itemize}
With the hope of avoiding data leakage, 
some of the existing works try to 
enhance privacy via MPC techniques 
such as HE and SS~\cite{vfl_lr,interactive}. 
Nevertheless, these works fail to 
address all the leakage cases. 
For instance, ~\citet{vfl_lr} propose to 
encrypt the derivatives in the backward propagation 
to avoid the label leakage from derivatives. 
However, they do not consider the label leakage 
from $Z_A = \xa\twa$ as shown in Figure~\ref{fig:limits}. 
Thus, \host can still make accurate guesses on the labels. 

There are also methods that try to 
add random noises to perturb these sensitive values. 
For instance, \citet{vfl_label_protect} propose to 
protect the labels by adding noises to $\nabla Z_A$. 
Nevertheless, we do not consider these methods in this work 
for three reasons. 
First, there is a tradeoff between 
protection ability and model quality. 
When strong privacy is required, 
the model accuracy will be harmed significantly. 
Second, these methods cannot provide a formal 
security guarantee like the MPC-based ones, 
which is expected in our work. 
Third, to the best of our knowledge, 
there is no such method that addresses 
all these leakage cases together.

In fact, the root cause of such potential data leakage 
is the design of local bottom models. 
For instance, as long as \host owns the bottom model, 
it is able to compute $Z_A$ anyway, 
leading to label leakage. 
To this end, the design of local bottom models 
cannot provide security guarantees 
under the ideal-real paradigm. 
How to build VFL models on non-outsourced data 
with provable security guarantees 
needs careful re-investigation.

\mysubsubsection{Our solution}
In order to support various kinds of features 
and guarantee privacy preservation in the meantime, 
this work draws the strengths of both categories. 
First, our work does not outsource the original datasets. 
In contrast, we follow the split learning paradigm 
to keep the private data inside each party. 
Second, we develop the federated source layers to 
collaboratively process the features of all parties. 
Unlike the bottom models, 
each party is not able to process the features individually 
in our design. 
Furthermore, our federated source layers 
achieve provable security guarantees. 
Putting them together, our work 
enjoys the flexibility to support various kinds of features 
whilst being privacy preserving.

%% file: sections/overview.tex
\section{Overview and Privacy Formulation}
\label{sec:method}

\subsection{Overview}
\label{sec:method_overview}

Figure~\ref{fig:overview} depicts 
the overview of \alg, 
which can be decoupled into two parts. 
First, the federated source layer 
works as the basic building block that 
unites the features of both parties. 
Given the output of the source layer, denoted as $Z$, 
a non-federated sub-module in \guest, 
namely the top model, 
plays the role of classifier/predictor{\footnote{
The top model can also be a federated module. 
However, a non-federated top model is more common 
in practice to address efficiency. 
Therefore, this work mainly focuses on 
non-federated top models 
whilst discusses how to adapt 
our federated source layers 
to federated top models 
in our appendix.}}.
During the backward process, 
\guest computes the loss function via 
the final predictions and ground truth labels, 
and then back propagates along the top model 
to obtain the backward derivatives $\nabla Z$. 
Finally, a federated procedure will be executed to 
update the model weights of the source layer. 
It is worthy to note that 
our framework differs from the model architecture of 
split learning. 
Our source layer requires all parties to 
collaboratively execute the learning process 
and outputs \textit{only} the aggregated results, i.e., $Z$. 
Whilst in split learning, each party 
can process the features individually and 
obtain the unaggregated values from its bottom model.

In practice, since the top model can be 
an arbitrary sub-module that 
minimizes the loss between predictions and labels, 
we concentrate on the source layers. 
Specifically, we consider two kinds of 
source layers in this work 
for different types of input features.

The first kind of source layers, 
namely \texttt{MatMul}, 
aims at the numerical feature values 
and computes a matrix multiplication 
in the forward propagation, i.e., 
$Z = \xa \twa + \xb \twb$, 
where $\twa, \twb$ are the model weights 
for \host and \guest, respectively. 
During the backward propagation, 
a federated procedure is executed to 
update $\twdot$ by the model gradients 
$\nabla\twdot = \xdot^T\nabla Z$ 
for each party.

For categorical features, 
we devise a more complex source layer 
called \texttt{Embed-MatMul}
that fuses the embedding lookup operation 
and matrix multiplication, i.e., 
$Z = \ea\twa + \eb\twb$, where 
$\edot = \lk(\qdot, \xdot)$ 
is the lookup operation given 
the embedding table $\qdot$. 
During the federated backward propagation, 
the backward derivatives and model gradients 
are computed as 
\begin{equation*}
	\nabla\twdot = \edot^T \nabla Z, \;
	\nabla\edot = \nabla Z \twdot^T, \;
	\nabla\qdot = \dlk(\nabla\edot, \xdot).
\end{equation*}

With these two kinds of source layers, 
we can derive various VFL models, 
including generalized linear models 
and neural networks. 
For instance, for logistic regression (LR), 
there is a \texttt{MatMul} source layer 
with $OUT=1$, 
whilst the top model adds the bias term 
and computes the sigmoid function, i.e., 
$\hat{y} = \textup{sigmoid}((\xa\twa+\xb\twb) + bias)$. 
For another example, 
as shown in Figure~\ref{fig:wdl_overview}, 
the wide and deep (WDL) model~\cite{wdl} 
consists of two source layers, 
one for the sparse, numerical features 
and the other for the categorical fields.

\subsection{Privacy to Matter}
\label{sec:privacy_matter}
As discussed in Section~\ref{sec:anatomy}, 
although datasets are kept local inside each party, 
the values generated in the learning process 
are also informative and would cause data leakage. 
Thus, before stepping into the design and analysis of 
the proposed source layers, 
we would like to discuss the privacy 
of all kinds of values, 
including forward activations, backward derivatives, 
model weights, and model gradients, respectively. 
In particular, we wish to conclude 
several guidelines about 
what contents must not be accessible to a specific party.

\textbf{Privacy of forward activations.\xspace}
Undoubtedly, forward activations have a strong relationship to 
the ground truth labels 
since they are learned to fit the labels. 
For instance, as illustrated in Figure~\ref{fig:limits}, 
in the \texttt{MatMul} source layer, 
$\xa\twa$ can be directly used to 
make predictions on the labels, 
so \host should have zero knowledge of them. 
Consequently, we make the requirement that \mytextcircled{1} 
{\hostul}\ul{\textit{ is not allowed to 
obtain any forward activations}}. 
However, as we have analyzed in Section~\ref{sec:anatomy}, 
since several existing VFL solutions 
fail to fulfill this requirement, 
\host can easily reveal a large fraction of labels, 
dampening the significance of privacy preservation. 

In addition to labels, forward activations 
inevitably contain sensitive feature information 
since they are originated from features. 
As aforementioned, 
\host is already prohibited from 
obtaining any forward activations, 
so we only need to consider whether \guest could 
guess the features of \host via forward activations. 
We divide forward activations into three kinds: 
(i) those solely dependent on $\xa$ 
(e.g., $\ea, \xa\twa$), 
(ii) those dependent on both $\xa,\xb$ 
(e.g., $Z, \hat{y}$), 
and (iii) those solely dependent on $\xb$ 
(e.g., $\eb, \xb\twb$). 
For the first kind, as discussed in Section~\ref{sec:anatomy}, 
we notice that they would 
disclose a certain level of feature similarity 
between different instances. 
For instance, in the \texttt{Embed-MatMul} layer, 
once \guest obtains $\ea$, it realizes whether two instances 
are equal on some categorical fields 
by comparing the embedding entries. 
Therefore, in order to protect the features, 
\mytextcircled{2}
\ul{\textit{we forbid any forward activations 
that are solely dependent on 
the features of }}{\hostul}\ul{\textit{ to be disclosed 
to }}{\guestul}. 
The second kind aggregates $\xa,\xb$ for the top model. 
Since the goal of VFL is to output the inference results 
to \guest, they should be accessible\footnote{
As described in Section~\ref{sec:method_overview}, 
our work can adapt to federated top models, 
where the second kind of forward aggregations 
are inaccessible. 
Thus, we focus on the federated source layer 
in this work and do not restrict \guest 
from them.}. 
Furthermore, when we analyze the security guarantees 
of our work in Section~\ref{sec:matmul_security} 
and Section~\ref{sec:embed_matmul_security}, 
we will formally prove that these values 
will not reveal the private data of \host. 
For the third kind, since they are independent on $\xa$, 
we do not make strict requirements. 
Instead, we will analyze whether they violate 
Req \mytextcircled{2} under specific scenarios. 
For instance, as we will discuss 
in Section~\ref{sec:matmul} and Section~\ref{sec:embed_matmul}, 
because \guest can derive $\xa\twa$ (or $\ea\twa$) 
if it gets access to $\xb\twb$ (or $\eb\twb$), 
we will restrict \guest from 
obtaining these forward activations 
to ensure the privacy of $\xa$ 
when designing our source layers.

\textbf{Privacy of backward derivatives.\xspace}
Similarly, in order to avoid the leakage to labels, 
\mytextcircled{3}
{\hostul}\ul{\textit{ is prohibited from 
accessing any backward derivatives}}, 
e.g., $\nabla Z, \nabla\ea$, 
since they are originated from 
the ground truth labels and prediction outputs. 
In Section~\ref{sec:expr}, we will empirically show that 
backward derivatives can disclose almost all labels to \host, 
which makes the privacy preservation in vain. 
Therefore, Req \mytextcircled{3}
is vital for designing the VFL algorithms. 

Although it is non-trivial to precisely extract 
the relevance between backward derivatives and features, 
we can define the security of backward derivatives 
according to forward activations. 
In fact, backward derivatives depicts the differences between 
forward activations and the ideal optimum. 
Thus, their informativeness regarding features bind together. 
Motivated as such, we make a symmetric requirement that 
\mytextcircled{4} 
\ul{\textit{if any forward activations are solely dependent  
on the features of }}{\hostul}\ul{\textit{, then 
the corresponding backward derivatives should also 
be kept secret from }}{\guestul}\ul{\textit{ as well}}.

\textbf{Privacy of model weights and model gradients.\xspace}
Although seemingly irrelevant to features or labels, 
the privacy of models is also meaningful. 
For one thing, given the fact that the values of model weights 
depict the feature importance upon the learning tasks, 
it will cause a leakage once a party knows 
the model information of the other party. 
For another, model gradients could cause leakage 
as discussed in recent studies
~\cite{deep_leak_grad,invert_grad}. 
Although these leakages sound to be task-specific, 
we still make a tough requirement that 
\mytextcircled{5} 
\ul{\textit{the model weights and gradients 
must be be hidden from the other party, 
including the signs and magnitudes of all coordinates}}, 
in order to avoid the potential risk of 
privacy leakage from models. 
Furthermore, \mytextcircled{6} 
\ul{\textit{we do not allow }}{\hostul}\ul{\textit{ to 
obtain the model weights or gradients of its own, 
even if the sign or magnitude of each coordinate}}. 
Otherwise, \host would infer the labels by analyzing 
its own feature contributions 
via the signs or magnitudes. 

Obviously, our privacy requirements 
address the possible leakage 
in split learning as discussed in Section~\ref{sec:anatomy}. 
They provide us with a template to 
derive the specific restrictions of each party 
when designing the federated source layers 
in Section~\ref{sec:matmul} and Section~\ref{sec:embed_matmul}.

%% file: sections/matmul.tex
\section{\texttt{MatMul} Federated Source Layer}
\label{sec:matmul}

Since matrix multiplication is one of the most 
essential arithmetics in ML, 
deriving a safe and accurate 
federated \texttt{MatMul} layer 
is vital to the VFL paradigm.

\subsection{Anatomy of Privacy Requirements}
To achieve promising privacy guarantees, 
we follow the analysis 
in Section~\ref{sec:privacy_matter} 
to derive what kinds of values would cause data leakage 
and make restrictions on them, 
i.e., these values should not be accessible to 
specific parties. 
We summarize the restrictions 
in Table~\ref{tb:matmul_restricts} 
and discuss the reasons below:
\begin{itemize}[leftmargin=*,label=$\triangleright$]
\item 
First, to avoid label leakage, 
\host must not get access to forward activations 
$Z$, $\xa\twa$, $\xb\twb$, 
backward derivatives $\nabla Z$, 
model weights, and model gradients of its own 
$\twa, \nabla\twa$, 
as analyzed in 
Req \mytextcircled{1}\mytextcircled{3}\mytextcircled{6}. 

\item
Second, as analyzed in Req \mytextcircled{2}, 
to avoid \host's feature leakage, 
\guest must not get access to $\xa\twa$, 
since they are merely the linear transformation of $\xa$, 
which would reveal a certain level of feature similarities. 
In fact, as $Z = \xa\twa + \xb\twb$, this also implies 
\guest must not get access to $\xb\twb$ and $\twb$. 

\item 
Finally, as analyzed in Req \mytextcircled{5}, 
we should restrict each party from obtaining 
the model weights or gradients of the other party 
$\tw_\otherdot, \nabla\tw_\otherdot$. 
\end{itemize}
Revisiting these restrictions, 
it is worthy to note that the existing split learning based 
approaches are insecure since the bottom models 
are typically $\twa, \twb$. 
It gives us a lesson that although features 
can be maintained locally, 
we shall not let each party directly 
learn a bottom model to process the feature individually. 
In the following subsections, 
we describe our algorithm protocol for 
the \texttt{MatMul} source layer 
and analyze the security guarantees.

\begin{table}[!t]
\small
\centering
\caption{\small{Summary of the restrictions for \texttt{MatMul}, i.e., the contents that each party must not get access to.}}
\begin{tabular}{|c|c|c|}
\hline
& \textbf{\host} & \textbf{\guest} \\
\hline
\hline
Model & 
\multicolumn{2}{c|}{$\twa, \twb$} \\
\hline
Forward & $Z, \xa\twa, \xb\twb$ & $\xa\twa, \xb\twb$ \\
\hline
Backward & $\nabla Z, \nabla\twa, \nabla\twb$ 
& $\nabla\twa$ \\
\hline
\end{tabular}
\label{tb:matmul_restricts} 
\end{table}

\begin{figure}[!t]
\centering
\includegraphics[width=3.4in]{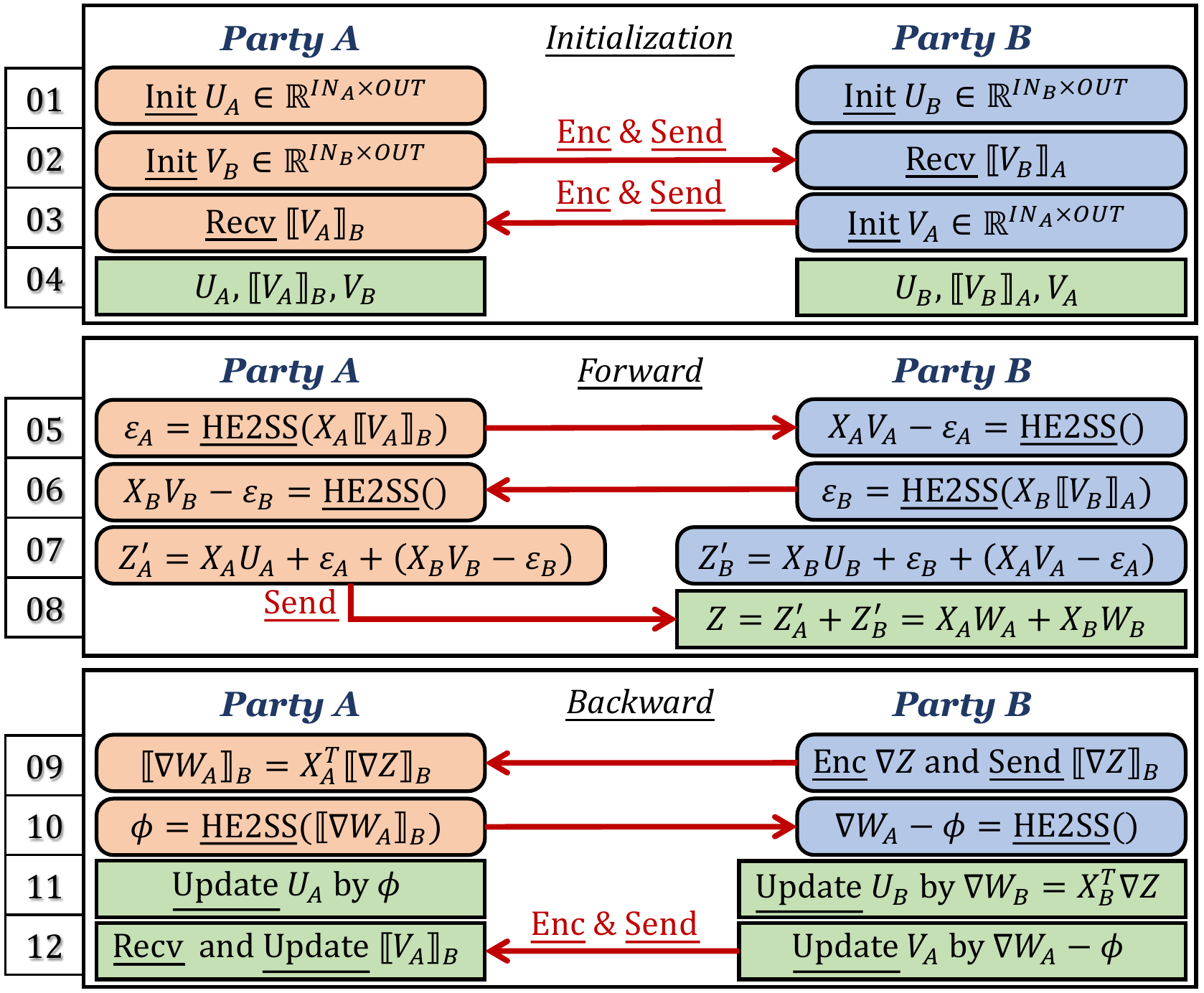}
\caption{{
Our \texttt{MatMul} source layer. 
All cross-party transmission (red arrows) 
are protected by HE or SS.}}
\label{fig:matmul_routine}
\end{figure}

\subsection{Algorithm Protocol}
The protocol of our \texttt{MatMul} 
source layer is presented 
in Figure~\ref{fig:matmul_routine} 
and the details are walked through below. 

\textbf{Initialization.\xspace}
To begin with, since both parties must not 
get access to the model weights, 
we follow the MPC-based methods to 
leverage the SS technique to 
break model weights onto both parties. 
To be specific, we secretly share the model weights by 
$\twdot = \udot + \vdot$, 
where $\udot$ and $\vdot$ are held by 
different parties so that 
none of them knows what 
$\twdot$ exactly is. 
To achieve this goal, on initialization, 
\host initializes $\ua$ for itself 
and $\vb$ for \guest (left hand side of Line 1-2). 
Furthermore, the encrypted version $\enca{\vb}$ 
is sent to \guest for future use (Line 3). 
\guest executes a symmetric routine 
to initialize $\ub, \va$.

\textbf{Forward propagation.\xspace}
Since the model weights are secretly shared, 
the results are also broken into four parts, i.e.,  
$Z=\xa\ua+\xa\va+\xb\ub+\xb\vb$. 
Among them, $\xdot\udot$ can be 
computed by one party alone, 
whilst $\xdot\vdot$ requires decryption. 
To be specific, there are three steps 
in the forward propagation. 
Since the first two steps are symmetric in both parties, 
we only describe those in \host. 
\begin{enumerate}[leftmargin=*]
\item[(1)] (Line 5-6)
\host computes  
$\encb{\xa\va}$ and 
transforms it into an SS variable 
$\langle \eps_A, \xa\va - \eps_A \rangle$; 
receives and decrypts 
the piece of sharing $\xb\vb - \eps_B$ from \guest. 
\item[(2)] (Line 7)
\host computes 
$Z_A^\prime = \xa\ua + \eps_A + (\xb\vb - \eps_B)$. 
\item[(3)] (Line 8)
Finally, 
\guest sums into $Z = Z_A^\prime + Z_B^\prime$. 
\end{enumerate}
It is worthy to note that 
all the random obfuscation values 
($\eps_A, \eps_B$) 
are eliminated to achieve the lossless property, i.e., 
\begin{equation*}
\begin{aligned}
	Z &= 
	(\xa\ua + \eps_A + (\xb\vb - \eps_B)) + 
	(\xb\ub + \eps_B + (\xa\va - \eps_A)) 
	\\
	&= \xa\ua + \xa\va + \xb\ub + \xb\vb
	= \xa\twa + \xb\twb.
\end{aligned}
\end{equation*}

\textbf{Backward propagation.\xspace}
For \guest, the model gradients can be 
computed via 
$\nabla\twb=\xb^T\nabla Z$ 
(right hand side of Line 11). 
For \host, we have to leverage the power of 
HE and SS since $\xa$ and $\nabla Z$ 
are kept secret by different parties. 
To be specific, we first send \host 
the encrypted derivatives $\encb{\nabla Z}$ 
to compute the encrypted model gradients $\encb{\nabla\twa}$, 
which will then be transformed into the SS variable 
$\langle \phi, \nabla\twa - \phi \rangle$ (Line 9-10). 
Furthermore, to prohibit any party 
from obtaining $\nabla\twa$ in plaintext, 
we do not restore the SS variable. 
Instead, we update the secretly shared model weights 
$\ua, \va$ in a complementary way (Line 11-12), i.e., 
\begin{equation*}
	(\ua - \eta \phi) + (\va - \eta (\nabla\twa - \phi))
	\Leftrightarrow \twa - \eta \nabla\twa.
\end{equation*}
Consequently, the algorithm 
accurately updates $\twa$ 
whilst guarantees that 
none of the parties gets access to $\nabla\twa$.

\subsection{Security Analysis}
\label{sec:matmul_security}
Obviously, our protocol 
in Figure~\ref{fig:matmul_routine} 
satisfies all the requirements 
in Table~\ref{tb:matmul_restricts}. 
All the informative values 
such as $\xdot\udot, \xdot\vdot, \nabla Z$ 
are protected by either HE or SS. 
Moreover, the protocol is lossless ---
both the forward outputs and backward updates 
are accurate. 
To be formal, we identify two ideal functionalities 
$\mathcal{F}_{\texttt{MatMulFw}}$ 
and $\mathcal{F}_{\texttt{MatMulBw}}$ 
for the forward and backward propagation, respectively. 
In $\mathcal{F}_{\texttt{MatMulFw}}$, 
each party inputs its own mini-batch data ($\xdot$) and 
the secretly shared and/or encrypted models 
($\udot, \enc{\vdot}_\otherdot$). 
\host outputs nothing whilst \guest outputs $Z$. 
In $\mathcal{F}_{\texttt{MatMulBw}}$, 
each party inputs its own mini-batch data 
and the secretly shared and/or encrypted models, 
and \guest further inputs $\nabla Z$. 
Then, each party outputs the updated models. 
Given these ideal functionalities, 
we provide the security guarantees of 
our \texttt{MatMul} source layer below.
\begin{theorem}
\label{thm:matmul_source}
The protocol of \texttt{MatMul} source layer 
securely realizes the ideal functionalities 
$\mathcal{F}_{\texttt{MatMulFw}}$ 
and $\mathcal{F}_{\texttt{MatMulBw}}$ 
in the presence of a semi-honest adversary 
that can corrupt one party.
\end{theorem}
Since the models are secretly shared and/or encrypted, 
they are secure in essence. 
However, recall that when the top model is non-federated, 
\guest obtains the input $Z$ 
and the output $\nabla Z$ of the top model, 
whose security is not analyzed 
by Theorem~\ref{thm:matmul_source}. 
To this end, we present Theorem~\ref{thm:matmul_top}.
\begin{theorem}
\label{thm:matmul_top}
Given $Z, \nabla Z$ in the \texttt{MatMul} source layer, 
there are infinite possible values for $\xa, \xa\twa$. 
\end{theorem}
Due to the space constraint, we defer the proofs to 
our appendix. 
Furthermore, we also analyze the security guarantees 
when the source layer 
is followed by a federated top model in our appendix.

%% file: sections/embed.tex
\section{\texttt{Embed-MatMul} Federated Source Layer}
\label{sec:embed_matmul}

To handle categorical features, 
a federated source layer that supports 
embedding lookup is desired. 
In this section, we introduce 
the \texttt{Embed-MatMul} source layer.

\subsection{Anatomy of Privacy Requirements} 

\begin{table}[!t]
\small
\centering
\caption{\small{Summary of the restrictions for \texttt{Embed-MatMul}, i.e., the contents that each party must not get access to.}}
\begin{tabular}{|c|c|c|}
\hline
& \textbf{\host} & \textbf{\guest} \\
\hline
\hline
Model & 
\multicolumn{2}{c|}{$\twa, \twb, \qa, \qb$} \\
\hline
Forward 
& $Z, \ea, \eb, \ea\twa, \eb\twb$ 
& $\ea, \eb, \ea\twa, \eb\twb$ \\
\hline
Backward 
& \specialcell{$\nabla Z, \nabla\ea, \nabla\eb$\\$\nabla\twa, \nabla\twb, \nabla\qa, \nabla\qb$} 
& \specialcell{$\nabla\ea, \nabla\eb$\\$\nabla\twa, \nabla\qa, \nabla\qb$}  \\
\hline
\end{tabular}
\label{tb:embed_restricts} 
\end{table}

\begin{figure}[!t]
\centering
\includegraphics[width=3.4in]{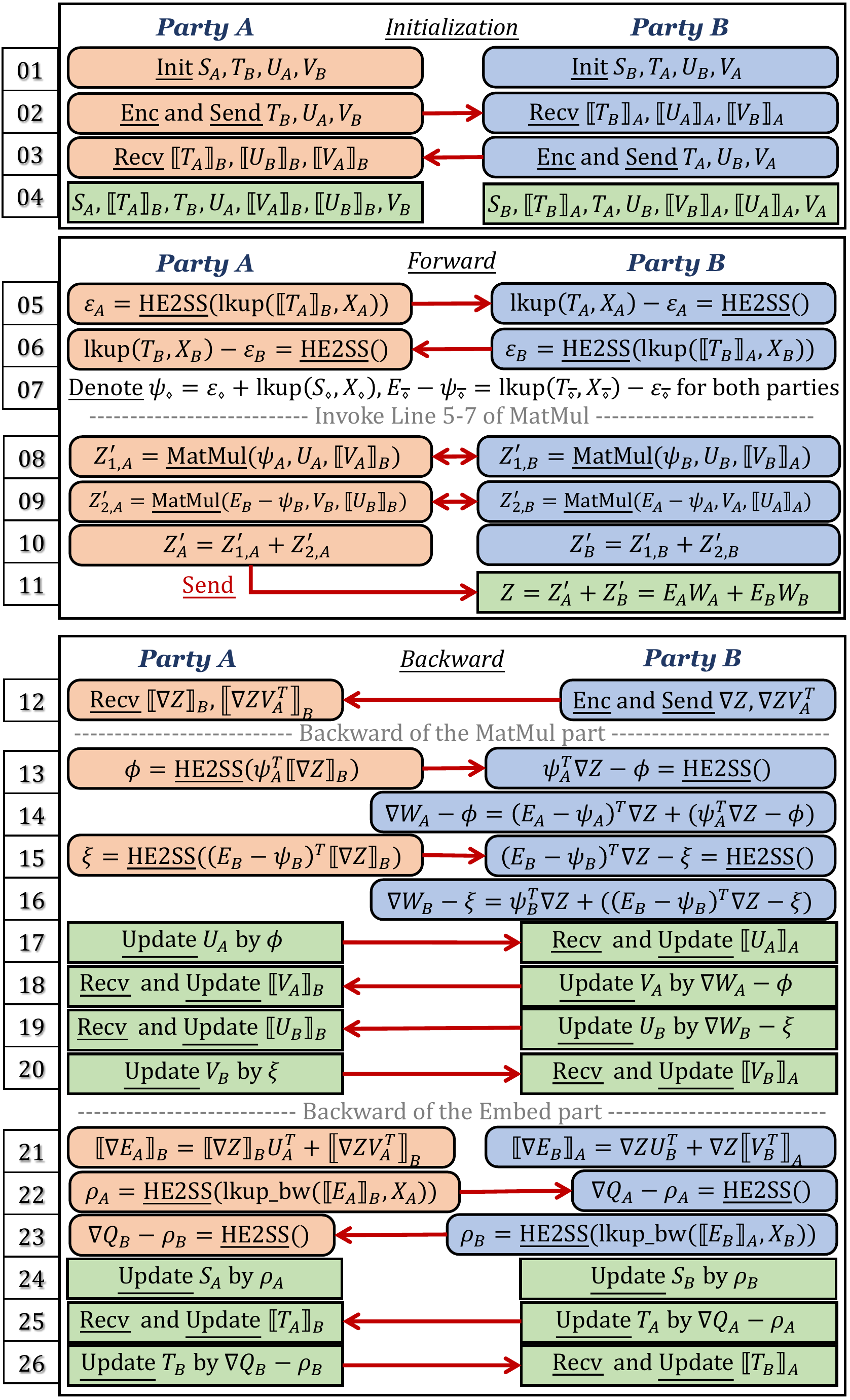}
\caption{{Our \texttt{Embed-MatMul} source layer. 
All cross-party transmission (red arrows) 
are protected by HE or SS.}}
\label{fig:embed_matmul_routine}
\end{figure}

Similarly, we first derive the restrictions 
according to our analysis in Section~\ref{sec:privacy_matter}. 
Table~\ref{tb:embed_restricts} summarizes 
these restrictions and 
the corresponding reasons are discussed below.
\begin{itemize}[leftmargin=*,label=$\triangleright$]
\item
First, the restrictions listed 
in Table~\ref{tb:matmul_restricts} 
should also be applied here. 

\item
Second, in order to avoid label leakage, 
\host must not get access to all forward activations 
($Z, \ea, \eb, \ea\twa, \eb\twb$), 
all backward derivatives 
($\nabla Z, \nabla\ea, \nabla\eb$), 
and the model weights and model gradients of its own 
($\twa,\qa,\nabla\twa,\nabla\qa$). 

\item	
Third, to avoid feature leakage, 
\guest must not get access to $\ea, \ea\twa$, 
since they are originated from $\xa$. 
As $Z = \ea\twa + \eb\twb$, we should also 
prohibit \guest from obtaining $\eb\twb$. 

\item 
Finally, it is worthy to note that we further 
prohibit \guest from obtaining the values 
related to its own embedding table, i.e., 
$\qb, \eb, \nabla\eb$. 
The reason is that since $\nabla\eb = \nabla Z \twb^T$, 
it is possible for \guest to infer $\twb$ 
once it gets access to those values. 
Hence, we make these strong restrictions 
to ensure the security. 
\end{itemize}
According to the above discussion, 
both parties are not allowed to 
obtain their own model weights, i.e., $\qdot, \twdot$. 
Again, this is contradictory to the design of bottom models 
in split learning. 
Consequently, we desiderate a new protocol 
for our \texttt{Embed-MatMul} source layer 
with promising security guarantees.

\subsection{Algorithm Protocol} 
Considering that both parties must not 
get access to the embedding tables, 
we apply SS to the embedding tables similarly, i.e.,  
$\qdot = \sdot + \tdot$, 
where $\sdot$ and $\tdot$ are managed by 
different parties. 
As a result, 
the embedding entires are also broken into two parts, i.e.,  
$\edot=\lk(\qdot,\xdot)=\lk(\sdot,\xdot)+\lk(\tdot,\xdot)$.
Therefore, each party cannot obtain $\edot$ alone 
since it has zero knowledge on $\tdot$. 
We present 
our \texttt{Embed-Matmul} federated source layer 
in Figure~\ref{fig:embed_matmul_routine} 
and describe the sketch of routines below.

\textbf{Initialization.\xspace}
Similar to the \texttt{MatMul} source layer, 
\host prepares $\sa, \ua$ for itself 
and $\tb, \vb$ for \guest, 
Then the encrypted SS pieces 
$\enca{\tb}$, $\enca{\ua}$, $\enca{\vb}$ 
are sent to \guest for future use. 
\guest executes a symmetric routine. 

\textbf{Forward propagation.\xspace}
We divide the forward process into two stages, 
for \texttt{Embed} and \texttt{MatMul}, respectively. 

The first stage retrieves 
the secretly shared embedding entries. 
Since embedding lookup requires 
the exact values of $\xdot$, 
we perform the lookup operation 
over the encrypted table $\enc{\tdot}_\otherdot$ 
in each party and convert it into an SS variable (Line 5-6). 
Finally, by combining with the lookup results on 
the rest piece $\sdot$, we successfully 
break the embedding entries in an SS manner, 
i.e., $\langle \psi_\mydot, \edot - \psi_\mydot \rangle$ 
(Line 7). 

The second stage 
performs two matrix multiplication, i.e., 
\begin{equation*}
\begin{aligned}
	\textup{(Line 8)}\;
	Z_{1,A}^\prime + Z_{1,B}^\prime &= \psi_A(\ua + \va) + \psi_B(\ub + \vb), \\
	\textup{(Line 9)}\;
	Z_{2,A}^\prime + Z_{2,B}^\prime &= (\eb - \psi_B)(\vb + \ub) + (\ea - \psi_A)(\va + \ua), 
\end{aligned}
\end{equation*}
using the same routine in the forward propagation of 
our \texttt{MatMul} layer (Figure~\ref{fig:matmul_routine}). 
Finally, the forward outputs 
can be computed via 
$Z = 
Z_{1,A}^\prime + Z_{1,B}^\prime + 
Z_{2,A}^\prime + Z_{2,B}^\prime = 
\ea\twa + \eb\twb$ 
(Line 10-11).

\textbf{Backward propagation.\xspace}
Next, we describe the backward process, 
which is also made up of two stages. 

The first stage takes in charge of 
the backward process of \texttt{MatMul}, 
which updates model weights $\twdot$ by 
$\nabla\twdot = \edot^T \nabla Z$. 
Similar to the backward process of 
the \texttt{MatMul} source layer, 
\guest encrypts $\nabla Z$ 
to protect the labels (Line 12). 
Upon receiving $\encb{\nabla Z}$, 
\host computes $\encb{\psi_A^T \nabla Z}$ 
via homomorphic arithmetics, 
which is then transformed into an SS variable, i.e., 
$\langle \phi, \psi_A^T \nabla Z - \phi \rangle$ 
(Line 13). 
By computing $(\ea - \psi_A)^T \nabla Z + 
(\psi_A^T \nabla Z - \phi) = 
\nabla\twa - \phi$ in \guest (Line 14), 
we secretly share $\nabla\twa$ 
onto both parties, 
i.e., $\langle \phi, \nabla\twa - \phi \rangle$. 
A similar routine works for $\nabla\twb$ as well 
(Line 15-16). 
Consequently, none of the parties gets access to 
$\nabla\twa$ or $\nabla\twb$, 
which obeys Table~\ref{tb:embed_restricts} 
so that the model weights can be updated accurately 
without any information leakage.

The second stage is for the embedding tables. 
Since the backward operation $\dlk(\cdot, \cdot)$ 
requires to know the exact values of $\xdot$, 
we determine to perform it over the encrypted 
backward derivatives $\enc{\edot}_\otherdot$. 
Therefore, 
both parties first compute 
$\encb{\nabla\ea}, \enca{\nabla\eb}$ 
via homomorphic arithmetics, respectively (Line 21), 
and then performs the backward operation 
$\enc{\nabla\qdot}_\otherdot = 
\dlk(\enc{\edot}_\otherdot, \xdot)$, 
which is eventually transformed into SS variables 
i.e., $\langle \rho_\mydot, \nabla\qdot - \rho_\mydot \rangle$ 
(Line 22-23). 
Finally, the secretly shared embedding tables 
are updated in the SS manner (Line 24-26).

\subsection{Security Analysis}
\label{sec:embed_matmul_security}
Indisputably, our protocol 
in Figure~\ref{fig:embed_matmul_routine} 
accomplishes all the privacy requirements 
in Table~\ref{tb:embed_restricts} 
and the results are lossless as well. 
To be formal, we identify two ideal functionalities 
$\mathcal{F}_{\texttt{EmbedMatMulFw}}$ 
and $\mathcal{F}_{\texttt{EmbedMatMulBw}}$ 
for the forward and backward propagation, respectively. 
In $\mathcal{F}_{\texttt{EmbedMatMulFw}}$, 
each party inputs the private mini-batch data ($\xdot$) 
and the secretly shared and/or encrypted models 
(e.g., $\sdot, \enc{\vdot}_\otherdot$). 
\host outputs nothing whilst \guest outputs $Z$. 
In $\mathcal{F}_{\texttt{EmbedMatMulFw}}$, 
each party inputs the private mini-batch data 
and the secretly shared and/or encrypted models, 
and \guest further inputs $\nabla Z$. 
Then, each party outputs the updated models. 
Then, We present the security guarantees of 
our \texttt{Embed-MatMul} source layer 
in Theorem~\ref{thm:embed_matmul_source}.
\begin{theorem}
\label{thm:embed_matmul_source}
The protocol of \texttt{Embed-MatMul} source layer 
securely realizes the ideal functionalities 
$\mathcal{F}_{\texttt{EmbedMatMulFw}}$ 
and $\mathcal{F}_{\texttt{EmbedMatMulBw}}$ 
in the presence of a semi-honest adversary 
that can corrupt one party.
\end{theorem}
Again, we further analyze the security of 
$Z, \nabla Z$ since they are released to \guest 
when the top model is non-federated, 
which is given in Theorem~\ref{thm:embed_matmul_top}.
\begin{theorem}
\label{thm:embed_matmul_top}
Given $Z, \nabla Z$ in the \texttt{Embed-MatMul} source layer, 
there are infinite possible values 
for $\xa, \qa, \twa, \ea, \ea\twa$. 
\end{theorem}
The proofs are deferred to the appendix 
due to the space constraint, 
where we further analyze the security 
when a federated top model follows 
our \texttt{Embed-MatMul} source layer.

%% file: sections/exper.tex
\section{Implementation and Evaluation}
\label{sec:expr}

\begin{figure}[!t]
\centering
\includegraphics[width=3in]{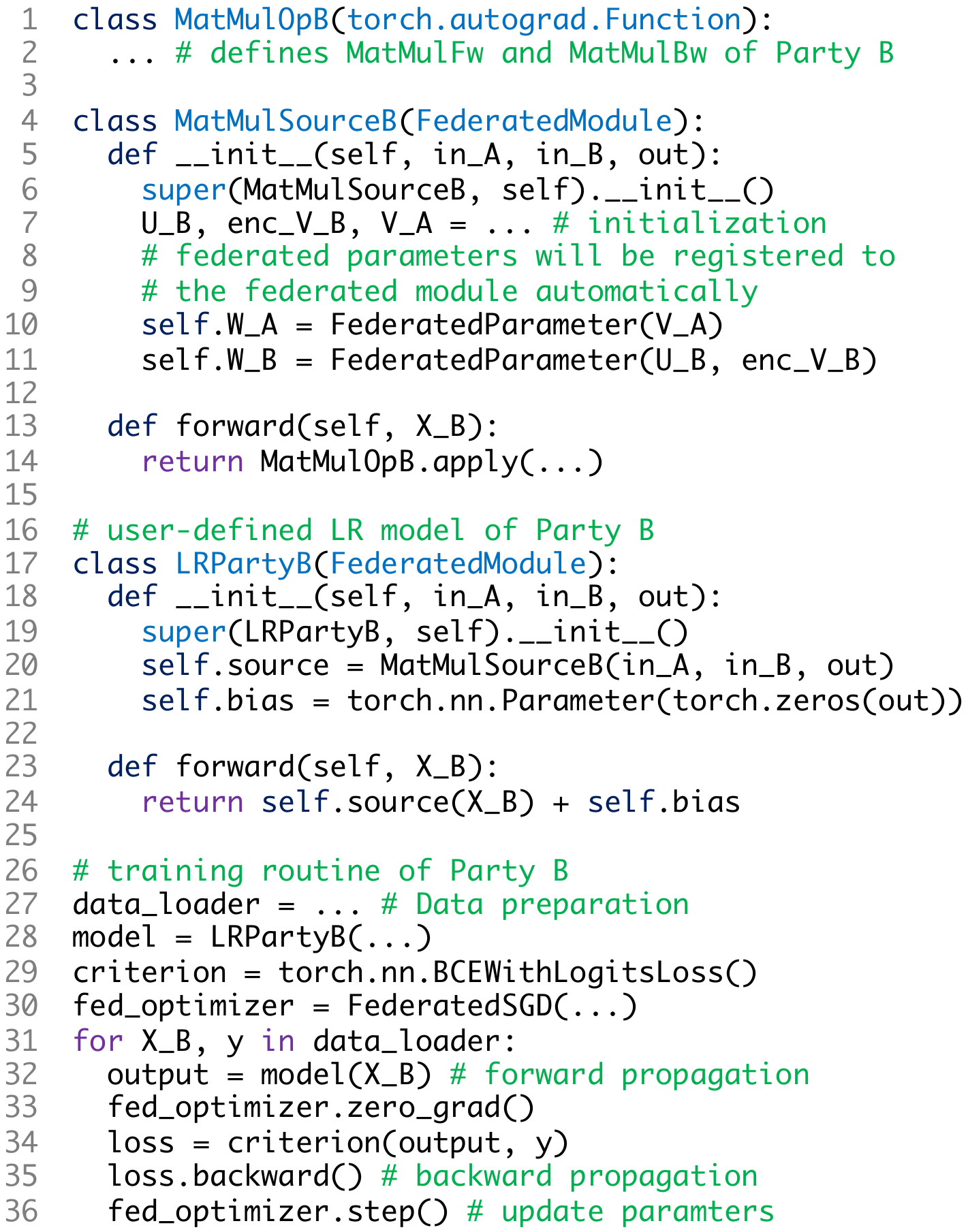}
\caption{{
Code snippets of the LR algorithm in \alg. 
}}
\label{fig:api}
\end{figure}

\subsection{Implementation and Experimental Setup}
\mysubsubsection{{Implementation}}
We implement \alg on top of \texttt{PyTorch} 
and the \texttt{GMP} library 
to be efficient and user-friendly. 

\textbf{Cryptography Acceleration.\xspace}
We employ the \texttt{GMP} library to 
develop an efficient library for 
the Paillier cryptosystem 
and support high-performance parallel processing 
with \texttt{OpenMP}. 
To vectorize the homomorphic operations, 
we introduce an abstraction called \texttt{CryptoTensor}, 
which supports fruitful primitives for 
both dense and sparse computation of encrypted tensors 
such as matrix multiplication and scatter addition. 

\textbf{\texttt{PyTorch} Integration.\xspace} 
As shown in Figure~\ref{fig:api}, 
\alg follows the API-style of \texttt{PyTorch} 
to be user-friendly. 
The forward and backward procedures 
of federated source layers 
are implemented as \texttt{autograd} operations 
to achieve automatic differentiation. 
All models are derived from \texttt{FederatedModule}, 
a wrapper for \texttt{PyTorch Module} that 
automatically registers the \texttt{FederatedParameter} 
(e.g., $\udot, \vdot$). 
The training routines are the same as 
non-federated learning, except that 
we define a \texttt{FederatedOptimizer} to 
update the secretly shared model weights. 
\alg has been integrated into the productive pipelines 
of our industrial partner 
and deployed to many real-world collaborative applications.

\mysubsubsection{{Experimental Setup}}
All experiments are conducted on two servers 
equipped with 96 cores, 
375GB of RAM, and 10Gbps of network bandwidth. 
For each experiment, we make five runs and 
report the mean and standard deviation. 

\textbf{Models.\xspace}
Since features should be partitioned under the VFL setting, 
whilst an image or a sentence would not be split 
and owned by different parties, 
the inputs of VFL models are usually tabular datasets 
rather than image or text datasets. 
Thus, we do not conduct experiments on CV or NLP models 
in this work. 
To be specific, We conduct experiments on 
five widely-used models, 
which are LR, multinomial LR (MLR), 
MLP, WDL~\cite{wdl}, and DLRM~\cite{dlrm}.

\textbf{Datasets.\xspace}
As listed in Table~\ref{tb:datasets}, 
we use six public datasets{\footnote{\url{
https://www.csie.ntu.edu.tw/~cjlin/libsvmtools/datasets}}} 
and one industrial advertising dataset. 
We evenly divide the features for the two parties. 
Since we focus on the algorithm protocols, 
we assume datasets are already aligned 
by the private set intersection (PSI) technique
~\cite{psi1,psi2,psi3}, 
which is a general data preprocessing 
in VFL~\cite{fl_concepts}. 
In Section~\ref{sec:related}, we will discuss how to 
extend our work when datasets cannot be aligned by PSI.

\textbf{Protocols.\xspace}
As we will show in Section~\ref{sec:expr_ability}, 
\alg achieves comparable model performance as 
non-federated learning on collocated datasets. 
Thus, we use the same set of hyper-parameters 
with learning rate as 0.05, batch size as 128, 
and embedding dimension as 8. 
For each model, we train for 10 epochs with momentum SGD, 
where the momentum value is set as 0.9.

\begin{table}[!t]
\small
\centering
\caption{{Description of datasets used in our experiments.}}
\begin{tabular}{|c|c|c|c|}
\hline
\textbf{Dataset} 
& \specialcell{\textbf{\#Instances}\\\textbf{(train/test)}}
& \specialcell{\textbf{\#Features \&}\\\textbf{Avg \#nnz}}
& \textbf{\#Classes} \\
\hline
\hline
\textit{a9a} & 32K/16K & 123, 14 & 2 \\
\hline
\textit{w8a} & 50K/15K & 300, 12 & 2 \\
\hline
\textit{connect-4} & 50K/17K & 126, 42 & 3 \\
\hline
\textit{news20} & 16K/4K & 62K, 80 & 20 \\
\hline 
\textit{higgs} & 8M/3M & 28, 28 & 2 \\
\hline
\textit{avazu-app} & 13M/2M & 1M, 14 & 2 \\
\hline
\textit{industry} & 100M/8M & 10M, 12 & 2 \\
\hline
\end{tabular}
\label{tb:datasets} 
\end{table}

\subsection{Privacy Preservation}
We first conduct experiments to empirically evaluate 
the robustness of privacy preservation of our work. 
Since the MPC-based methods 
can achieve provable security guarantees, 
we only compare with the split learning based methods 
in this section.

\begin{figure}[!t]
\centering
\includegraphics[width=3.3in]{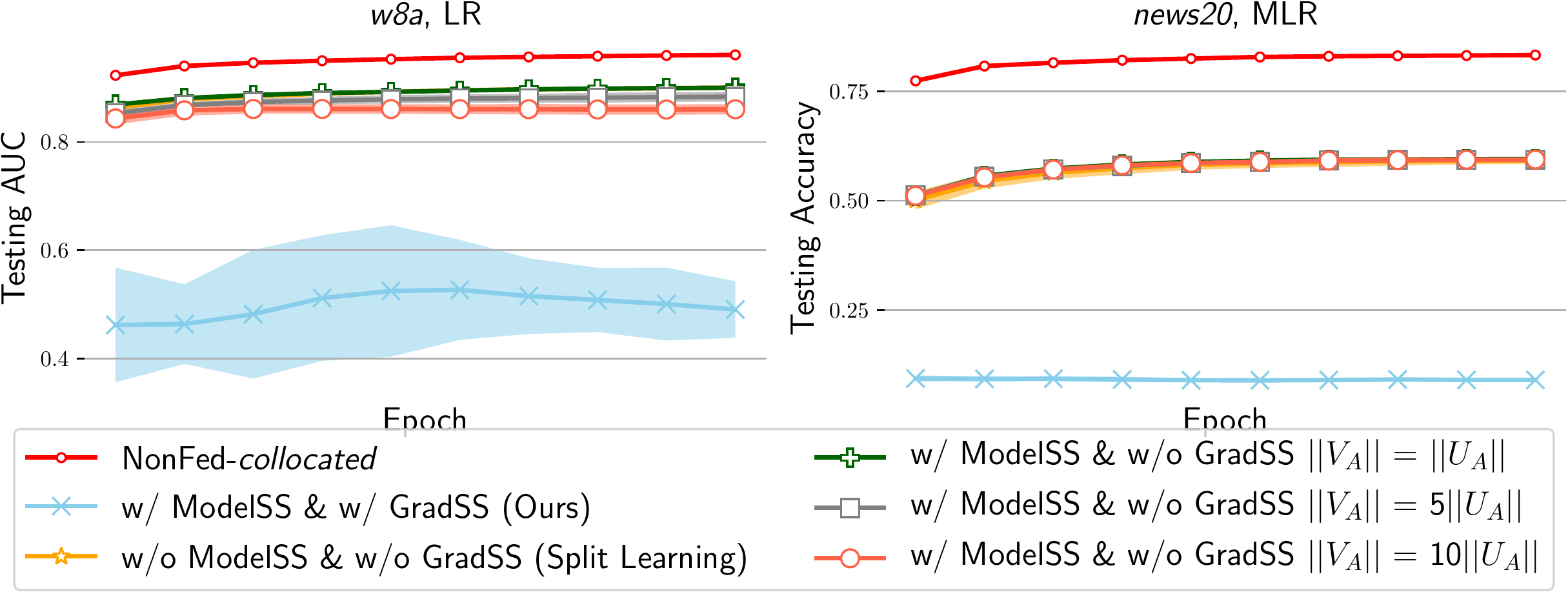}
\caption{{
Predicting labels with $\xa\twa$ or $\xa\ua$. 
We run experiments with or without the SS technique 
on model weights and gradients 
(denoted as ModelSS and GradSS). 
}}
\label{fig:adv1}
\end{figure}

\begin{figure}[!t]
\centering
\includegraphics[width=3.3in]{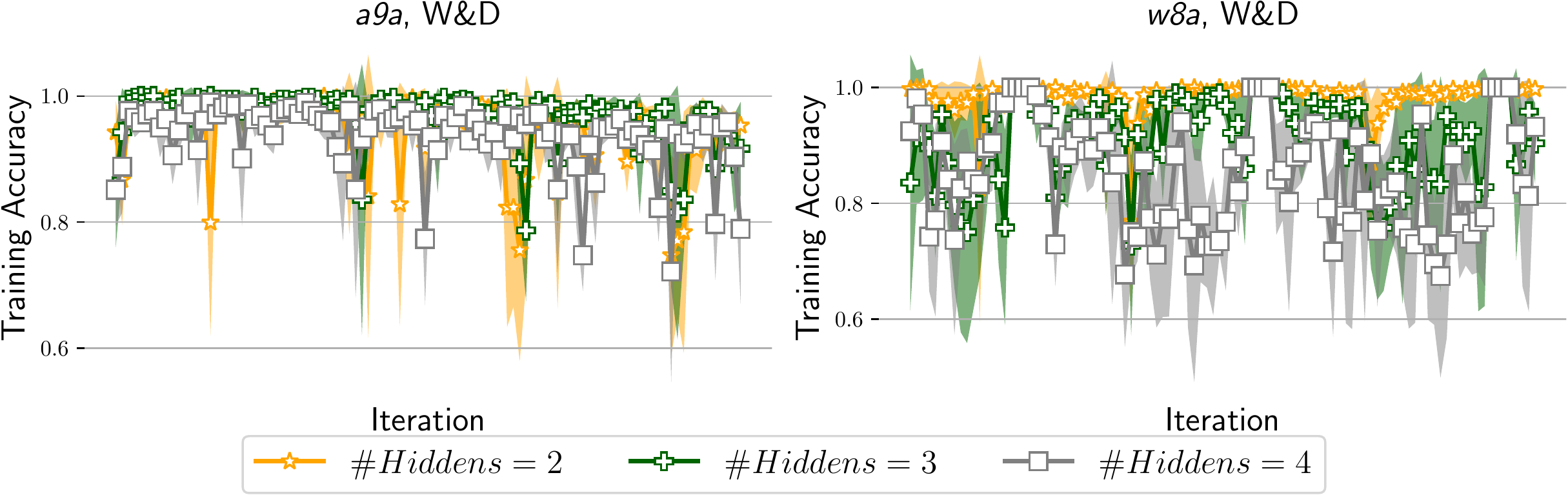}
\caption{{
Predicting labels with $\nabla\ea$. 
We vary the number of hidden layers 
after the embedding table.
}}
\label{fig:adv2}
\end{figure}

\begin{figure}[!t]
\centering
\includegraphics[width=3.3in]{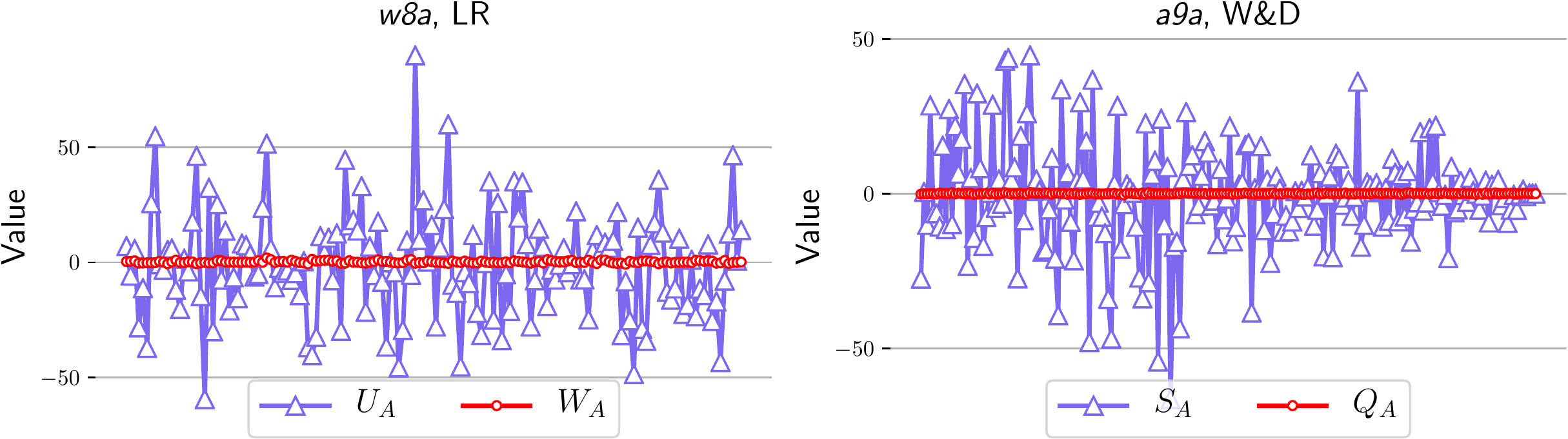}
\caption{{
The comparison of model weights ($\twa, \qa$) 
and the SS pieces ($\ua, \sa$) 
on the same coordinate. 
}}
\label{fig:model_ss}
\end{figure}

\mysubsubsection{{Protection for Forward Activations}}
As discussed in Section~\ref{sec:privacy_matter}, 
once \host gets access to the activations, 
there is a potential label leakage. 
However, many existing works allow \host to manage $\twa$ 
when implementing the LR or MLR models
~\cite{vfl_lr,async_vfl,private_vlr}. 
To empirically evaluate the leakage problem, 
we implement LR and MLR 
following the split learning paradigm 
and try to see how accurate it will be 
if \host predicts the labels via $\xa\twa$. 
For our work, we let \host predict with $\xa\ua$. 
The results are shown in Figure~\ref{fig:adv1}. 

When \host owns $\twa$, 
although the AUC or accuracy metrics 
of predicting with $\xa\twa$ 
are lower than those of predicting with $\xa\twa + \xb\twb$,
it is still unsatisfactory --- 
\host is able to make accurate predictions 
for a large portion of data. 
For instance, \host predicting with $\xa\twa$ 
on the \textit{w8a} dataset 
achieves around 0.9 of AUC on the testing set, 
which is only 0.06 lower than 
that of predicting with $\xa\twa + \xb\twb$.
In contrast, by using our \texttt{MatMul} source layer 
(Figure~\ref{fig:matmul_routine}), 
the predictions with $\xa\ua$ are purely random guesses 
(the AUC metric is around 0.5). 
It proves that the label information 
can be protected from the forward activations well 
with our \texttt{MatMul} source layer 
since we do not allow \host to manage a local bottom model. 

Readers might wonder the necessity of 
applying SS to model gradients. 
In other words, 
if we have already sheltered model weights via 
$\twa = \ua + \va$ on initialization, 
can we let \host get access to $\nabla\twa$ 
and update $\ua$ directly, 
rather than updating both $\ua,\va$ in the SS manner? 
To answer this question, we conduct experiments where 
$\twa$ is secretly shared on initialization 
but \host updates $\ua$ via $\nabla\twa$. 
Furthermore, we amplify the scale of $\va$ 
to better obscure $\twa$ from \host. 
Nevertheless, as shown in Figure~\ref{fig:adv1}, 
even though the model weights are secretly shared 
on initialization, \host can still obtain accurate predictions. 
Increasing the scale of $\va$ is of no use 
given the slight drop in the AUC/accuracy metrics. 
This is reasonable because if $\va$ keeps unchanged 
throughout the training phase, 
we can assume $\mathbb{E}[\xa\va] = C$, 
where $C$ is some constant, 
hence we have 
$\mathbb{E}[\xa\ua]=\mathbb{E}[\xa\twa] - C$. 
As a result, \host is still capable of 
making biased predictions, 
which inevitably leads to the label leakage.

To conclude, compared with the existing works 
that follow the local bottom model design, 
\alg is robust to any kind of adversaries that 
try to learn the private data via forward activations, 
thanks to our federated source layers.

\mysubsubsection{{Protection for Backward Derivatives}}
As discussed in Section~\ref{sec:privacy_matter}, 
the backward derivatives are very informative 
and would leak the training labels. 
To verify this, we implement 
the WDL model following the split learning paradigm, 
i.e., \host owns the embedding table $\qa$ 
in the bottom model 
and obtains $\nabla\ea$ in the backward propagation, 
which is done in several existing works
~\cite{interactive,fdml,vfl_label_protect}. 
Then, we let \host predicts the labels via $\nabla\ea$. 

The results in Figure~\ref{fig:adv2} reveals a horrible fact --- 
\host accurately predicts the labels of 
almost the entire training datasets 
via the derivatives $\nabla\ea$. 
Undoubtedly, this disobeys the goal of privacy preservation. 
Although it focuses on the training data, 
\host could fit a new model with $\xa$ and the leaked labels, 
and utilize the new model to make predictions for more data, 
leading to a more severe level of label leakage. 
Moreover, we note that 
this method attacks successfully regardless of 
how far $\nabla\ea$ are away from the labels 
(i.e., the number of hidden layers between the embedding table 
and the loss function). 

Theoretically speaking, for binary-classification tasks, 
the backward derivatives for positive and negative instances 
ought to have opposite directions 
since they contribute oppositely to the model. 
Thus, \host can utilize this property 
and compute the cosine similarity of two derivatives 
to see whether they have an opposite direction. 
By doing so, \host achieves an incredible success rate. 
It gives us a lesson that any algorithms 
allowing \host to obtain any backward derivatives 
would be vulnerable from label leakage. 

On the contrary, 
we thoroughly analyze the security of backward derivatives 
and propose corresponding privacy requirements 
to forbid data leakage from backward derivatives. 
For instance, in our \texttt{Embed-MatMul} source layer, 
\host does not own the embedding table in plaintext 
and only gets access to the encrypted derivatives 
$\encb{\nabla\ea}$, 
so \alg is robust to any adversaries that try to 
peek the private data through backward derivatives.

\mysubsubsection{{Protection for Models}} 
As we have discussed in Section~\ref{sec:privacy_matter}, 
the magnitudes or signs of model weights or gradients express 
the feature importance upon the learning tasks, 
so we forbid the models to be obtained. 
To achieve this goal, 
we expect the SS pieces of model weights and gradients 
do not reveal the magnitudes or signs of the ground truth values.

To illustrate the effectiveness of 
our protection for models, 
we plot the values of SS pieces and the ground truth values 
in Figure~\ref{fig:model_ss}. 
Due to the space constraints, 
we only plot the model weights $\twa,\sa$ 
and the sharing pieces $\ua,\qa$, 
whilst similar results are observed 
for other model weights and gradients. 
Figure~\ref{fig:model_ss} shows that 
the difference 
on each coordinate is random and sufficiently large 
so that both the magnitudes or signs of 
the ground truth values are inaccessible. 
As a result, each party cannot infer any information 
by analyzing the feature importance.

\begin{table}[!t]
\small
\centering
\caption{\small{
Averaged training time cost of one mini-batch (in seconds). 
We only record the time cost of matrix multiplication 
for a fair comparison. 
The standard deviation for all numbers 
are smaller than 10\% of the mean. 
SecureML without client aided fails to 
support the high-dimensional datasets 
\textit{news20}, \textit{avazu-app}, and \textit{industry}. 
}}
\begin{tabular}{|c|c|c|c|c|}
\hline
\multirow{3}*{\specialcell{\textbf{Dataset \&}\\\textbf{Sparsity}}}
& \multirow{3}*{\textbf{Model}} 
& \multicolumn{3}{c|}{\textbf{Time Cost/Batch (in seconds)}} 
\\
\cline{3-5}
& & \textbf{\alg} & \textbf{SecureML} 
& \specialcell{\textbf{SecureML}\\\textbf{\scriptsize{(Client-aided)}}}
\\
\hline
\hline
\textit{a9a} (88.72\%) 
& LR 
& 0.018 
& 0.567
& 0.003
\\
\hline
\textit{w8a} (96.12\%) 
& LR 
& 0.021
& 1.214
& 0.016
\\
\hline
\textit{connect-4} (66.67\%) 
& MLP 
& 1.114
& 5.703
& 0.008
\\
\hline
\textit{higgs} (Dense) 
& LR 
& 0.066 
& 0.178 
& 0.002
\\
\hline
\textit{news20} (99.87\%) 
& MLR 
& 1.817 
& $>$ 1800
& 0.364
\\
\hline
\textit{avazu-app} (99.99\%) 
& LR 
& 0.038 
& OOM
& 4.727
\\
\hline
\textit{industry} (99.99\%) 
& LR 
& 0.034 
& OOM
& 47.083
\\
\hline
\end{tabular}
\label{tb:efficiency} 
\end{table}

\begin{figure*}[!t]
\centering
\includegraphics[width=6.8in]{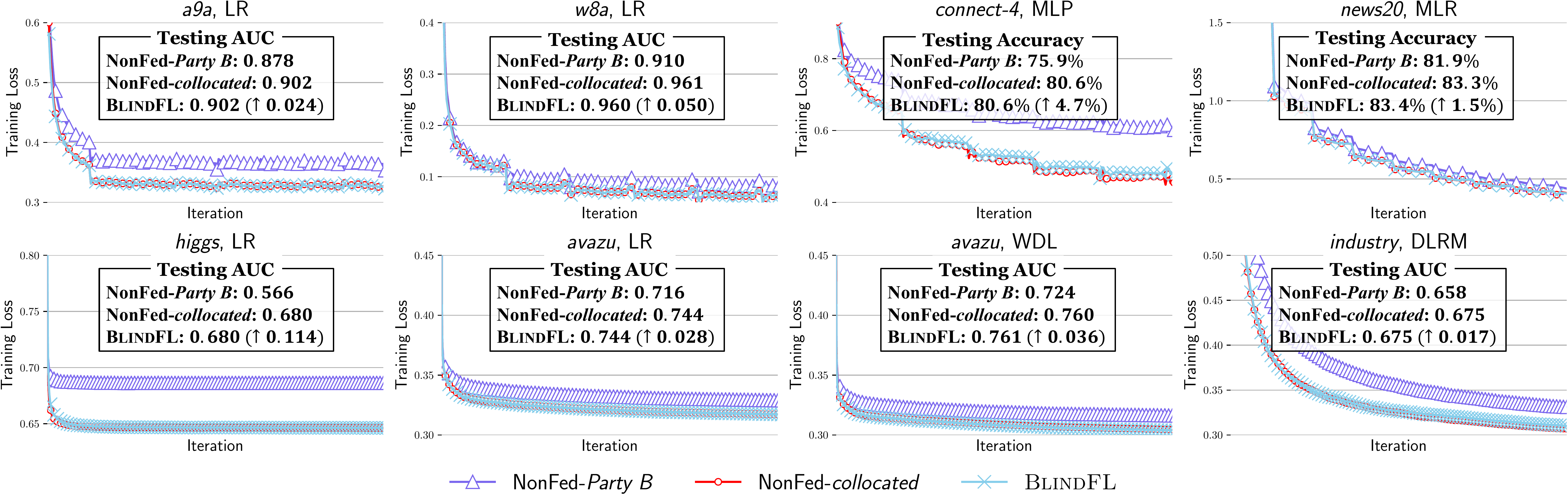}
\caption{{Training loss in terms of iterations and testing AUC or accuracy metrics. The standard deviation for all evaluation metrics are smaller than 5\% of the mean. The model performance of \alg is comparable to that of non-federated learning on collocated datasets, and better than that of non-federated learning on \guest's features only.}}
\label{fig:converge}
\end{figure*}

\subsection{Efficiency}
Apart from \alg, the MPC-based methods 
can also achieve robust security guarantees 
via data outsourcing. 
However, unlike these methods, 
our work is able to utilize the sparsity of features 
to speed up the computation. 
To verify this, 
we conduct experiments to 
compare the efficiency of \alg and 
the existing MPC-based counterparts. 
In particular, we consider SecureML~\cite{secureml}, 
which secretly shares the model weights and feature values 
via ABY~\cite{aby}, 
and performs matrix multiplication 
with the help of Beaver triplets. 
Moreover, as introduced by ~\citet{secureml}, 
SecureML has a client-aided variant 
that offloads the generation of Beaver triplets 
to a non-colluding third party 
so that no cryptographic operation will be involved 
in the ML tasks. 
We compare \alg with both variants of SecureML 
in terms of the training time cost. 
Furthermore, we only record the time cost of 
matrix multiplication and exclude 
that of non-linear activation functions (e.g., Sigmoid, ReLU) 
for a fair comparison. 
The results are given in Table~\ref{tb:efficiency}.

\mysubsubsection{\alg vs. SecureML}
We first discuss the performance without client aided 
on different kinds of datasets, respectively.

First, although SecureML can support 
the low-dimensional and sparse datasets, i.e.,  
the \textit{a9a}, \textit{w8a}, and \textit{connect-4} datasets, 
it cannot take the sparsity into account, 
leading to umpteenth redundant computations 
on the zero values. 
On the contrary, by keeping the data inside each party, 
\alg supports sparse matrix multiplication 
and therefore outperforms SecureML by a large extent. 
Furthermore, \alg gains more improvements 
when the datasets are sparser, 
and can be over 50$\times$ faster than SecureML.

Second, for the \textit{higgs} dataset, 
which is low-dimensional and dense, 
the gap between \alg and SecureML is smaller 
since there are no zero values. 
Nevertheless, our work still performs better and 
gives about 5$\times$ of acceleration. 
The reason is that we keep an encrypted version 
of $\enc{\vdot}_\otherdot$ on \textit{Party} $\mydot$ 
to reduce an extra communication round, 
whilst SecureML needs to 
generate new Beaver triplets 
for each iteration. 
Thus, our \texttt{MatMul} source layer 
also supports dense features well.

Third, SecureML fails to support 
the \textit{news20}, \textit{avazu-app}, 
and \textit{industry} datasets 
due to their high dimensionalities 
(either fails to accomplish the computation within 
a reasonable time or runs out of memory). 
In contrast, \alg handles 
these high-dimensional and sparse datasets well 
since we do not need to store or process the zero values. 
To the best of our knowledge, 
none of the existing works are able to support 
such a scale of datasets 
whilst guarantee promising privacy.

\mysubsubsection{\alg vs. Client-aided SecureML}
The client-aided SecureML runs much faster 
when the dimensionality is not high 
as no cryptography operations are needed, 
whilst \alg cannot be accelerated in this way. 
Nevertheless, for the \textit{avazu-app} and \textit{industry} 
datasets, which are extremely high-dimensional and sparse, 
\alg outperforms the client-aided SecureML. 
This is reasonable since SecureML has to 
process all the dimensions 
whilst \alg only needs to consider the non-zero ones. 
Furthermore, since there is no such 
a non-colluding third party in practice, 
the MPC-based methods would be even slower. 
As a result, \alg is more suitable for the sparse datasets 
in real-world VFL applications.

To summarize, \alg can outperform 
the MPC-based counterparts 
by an order of magnitude on sparse datasets 
and supports datasets with higher dimensionalities. 
Moreover, as we will illustrate 
in Section~\ref{sec:expr_ability}, 
\alg can also support the embedding lookup operation 
for categorical inputs, 
whilst there are no MPC-based methods that are designed 
for the embedding lookup operations to the best of knowledge. 
As a result, \alg is more powerful as 
it gains much better efficiency and functionality.

\mysubsubsection{Scalability}
In general, the time cost of the federated source layers 
dominates the total running time 
since the cryptography operations are time-expensive 
whilst the top model is a non-federated module. 
Thus, the running time of our work 
is proportional to the output dimensionality 
of the source layer, 
and does not change significantly 
w.r.t. the number of layers. 
Due to the space constraint, 
we defer the experimental results in our appendix 
and do not discuss the details here.

\subsection{Model Ability}
\label{sec:expr_ability}
As introduced in Section~\ref{sec:intro}, 
the VFL models are expected to be lossless, 
i.e.,  
the performance of VFL models should be 
(i) better than non-federated learning on 
the features of \guest only 
(denoted as NonFed-\guest); 
and (ii) comparable to non-federated learning on 
collocated features of both parties 
(denoted as NonFed-\textit{collocated}). 
To assess the lossless property of \alg, 
we conduct experiments on extensive datasets and models 
and show the results in Figure~\ref{fig:converge}. 

On all experiments, 
the convergence of \alg is consistent to 
that of NonFed-\textit{collocated} 
and is better than that of NonFed-\guest. 
The evaluation metrics on testing sets 
have similar observations as well --- 
the testing AUC or accuracy metrics 
of \alg and NonFed-\textit{collocated} 
are comparable, which are better than 
those of NonFed-\guest 
with the help of the extra features from \host. 
The experimental results verify 
the lossless property of \alg. 
In practice, it is forbidden to collect 
the data from different parties 
together for ML training or inference. 
As a result, \alg is more superior 
as it unites the features in different parties 
to enhance the model ability 
whilst guarantees the data privacy 
of all parties.

\mysubsubsection{Summary}
To conclude, \alg outperforms the existing works 
by providing more robust privacy guarantees 
and faster running speed. 
Furthermore, \alg achieves comparable 
model ability as collocated learning 
for a variety of models and datasets. 
Consequently, \alg is a good fit for 
a wide range of VFL applications.

%% file: sections/related.tex
\section{Related Works}
\label{sec:related}

\textbf{Vertical FL (VFL).\xspace}
With the ever-evolving concerns on data privacy, 
how to build ML models over different data sources 
with privacy preservation s gaining more popularity. 
VFL considers the case where different parties 
jointly train ML models 
over the partitioned features, which fits 
numerous real-world cross-enterprise 
collaborations~\cite{fl_concepts,secureboost,vfboost,privacy_vfl_tree,vfl_lr}. 
We categorize the existing works into two lines 
based on how they process the private features.

\textbf{The MPC-based methods.\xspace}
Since cryptographic methods 
can deliver promising security guarantees, 
it is a straightforward idea to apply them 
to the private data directly. 
A number of works encrypt features via HE 
and feed the encrypted features to ML models
~\cite{paillier_outsource_lr,cryptonets,cryptodl,sealion,private_bp_clound}.
To perform both addition and multiplication, 
these works adopt somewhat HE or fully HE, 
which are extremely time-consuming. 
Other works adopt SS to protect the data. 
~\citet{sharemind} proposed a data analysis framework 
based on additive SS. 
\citet{epic} trained the support vector machine models 
for image classification. 
SecureML~\cite{secureml} supports various ML models 
via the ABY scheme~\cite{aby}. 
ABY3~\cite{aby3} extends ABY to the three-party scenarios. 
SecureNN~\cite{securenn1,securenn2} optimizes for three-party neural networks. 
Nevertheless, although these works convey considerable 
privacy guarantees, the generic MPC framework involves 
sophisticated protocols and 
time-consuming computation primitives 
to achieve zero knowledge disclosure. 
Worse still, outsourcing the datasets is not suitable 
for sparse or categorical features. 

\textbf{Split learning.\xspace}
Rather than data outsourcing, 
the split learning paradigm~\cite{split_fed,split_learning} 
keeps the private data inside each party 
and leverages a local bottom model to 
process the features. 
Many VFL algorithms are designed in this way
~\cite{vfl_lr,interactive,vfl_label_protect,fdml,async_vfl}. 
Since features are visible to the owner party, 
it is very flexible to support various kinds of features. 
However, as discussed in Section~\ref{sec:anatomy}, 
the design of local bottom models faces several 
data leakage problems and fails to convey 
provable security guarantees. 
Although some existing works try to apply some 
cryptography operations to enhance privacy, 
they still suffer from data leakage. 
For instance, ~\citet{vfl_lr} tried to 
avoid label leakage from derivatives via HE and SS. 
However, \guest can still infer the labels  
as depicted in Figure~\ref{fig:limits}. 
In addition, 
noisy perturbation (or differential privacy) 
is also utilized to blur 
the intermediate results~\cite{vfl_label_protect}. 
Nevertheless, adding noises 
inevitably affects the model accuracy. 
In practice, it is non-trivial to figure out 
a suitable amount of noises 
that conveys considerable security whilst 
produces desirable model accuracy. 
Furthermore, all these works cannot provide 
security guarantees under the ideal-real paradigm 
as the MPC-based methods.

\textbf{Data alignment assumption.\xspace}
Although most of the VFL algorithms assume the instances 
of different parties have been aligned via 
the PSI technique~\cite{psi1,psi2,psi3}, 
there are also works that focus on 
the cases where data cannot be aligned due to 
stronger privacy requirements~\cite{apsi,psu}. 
However, our work can be extended to these cases easily. 
For instance, \citet{apsi} proposed that 
only \guest can obtain the intersection whilst \host cannot. 
Then, for instances outside the intersection, 
\guest sets their derivatives as zeros 
so that model gradients will not be affected. 
This method can be applied to our work 
by tweaking Line 9 of Figure~\ref{fig:matmul_routine} 
and Line 12 of Figure~\ref{fig:embed_matmul_routine}. 
For another example, in~\cite{psu}, both parties can only 
obtain the union instead of intersection. 
Then, they propose to generate synthetic 
features and labels for instances outside the intersection. 
Obviously, this technique can also be integrated with our work 
during the data preprocessing phase.

\textbf{Reconstruction-based inference attacks.\xspace}
There are also many works that try to 
break the privacy of ML models 
via more advanced attacks. 
For instance, membership inference attacks 
(MIAs)~\cite{mia,mia_survey} 
try to infer whether an instance was used to 
train a given model. 
However, in our work, since the federated source layers 
are not released in plaintext, 
none of the parties (or any other attackers) 
can apply MIAs to attack the models. 
Deep leakage from gradients (DLG) is also 
a popular kind of 
attack methods~\cite{deep_leak_grad,invert_grad} 
which recover the private data from model gradients. 
However, model gradients are not disclosed to 
others in our algorithm protocol.

\textbf{Horizontal FL (HFL).\xspace}
Besides VFL, HFL~\cite{
konevcny2016federated_opt,konevcny2016federated_learn,McMahan2017_fl} 
is also one of the most popular fields of FL, 
where all parties 
have disjoint instance sets for the same features, 
i.e., the datasets are horizontally partitioned. 
There has been a wide range of works 
studying HFL~\cite{fedprox,dp_sage,google_fl_system}. 
The most essential privacy-preserving technique 
in this field is differential privacy (DP)
~\cite{dp_dl,dp_perf}, 
which is adopted to obscure the model weights 
or gradients so that the individual data 
become indistinguishable. 
Our work differs from these works in objection 
since we consider the VFL scenario. 

\textbf{Other related studies.\xspace}
Beyond FL, there is also an arousing interest in
various kinds of federated computing. 
For instance, 
the private set intersection technique~\cite{psi1,psi2,psi3} 
secretly extracts and aligns the joint set of private tables 
in different parties. 
Apart from the ``join'' operation, 
many other operations over private databases
are studied~\cite{privacy_db_op,opaque,privacy_db_xor}. 
Achieving privacy preservation from the hardware perspective 
is also a popular topic. 
With the help of the trusted execution environments (TEEs) 
such as Intel SGX and AMD memory encryption~\cite{sgx,amd}, 
we can get rid of umpteenth expensive cryptographic arithmetics 
when implementing privacy-preserving ML algorithms~\cite{oblivious_ml}. 
Nevertheless, the available memory for TEE is 
too small to support a large data volume 
and it requires extra hardware supports.

%% file: sections/conc.tex
\section{Conclusion and Future Directions}
\label{sec:conc}

This work proposed \alg, a brand new VFL framework. 
To address the functionality of VFL models, 
we designed the federated source layers to 
unite the data from different data sources. 
To protect data privacy, we analyzed 
the privacy requirements in-depth 
and carefully devised safe and accurate 
algorithm protocols. 
Experimental results show that 
\alg is able to 
support various kinds of features efficiently 
and achieve promising privacy guarantees. 

Beyond this work, we wish to strengthen \alg 
in two directions. 
First, it is worthy to study 
how to extract feature interactions 
between parties to support more ML models 
(e.g., factorization machines). 
Second, how to apply adaptive optimizers (e.g., Adam) 
to secretly shared model gradients 
is also an interesting topic. 
We will leave the exploration of 
these topics as our future works.

%% file: sections/appendix.tex
\newpage

\section{Security Analysis}
\label{sec:proof}

In this section, we provide the definition of 
our ideal functionalities 
and proofs for the theorems 
in Section~\ref{sec:matmul_security} 
and Section~\ref{sec:embed_matmul_security}.

\subsection{Proofs for Section~\ref{sec:matmul_security}}

\textsc{Theorem~\ref{thm:matmul_source}.}
{\em
The protocol of \texttt{MatMul} source layer 
securely realizes the ideal functionalities 
$\mathcal{F}_{\texttt{MatMulFw}}$ 
and $\mathcal{F}_{\texttt{MatMulBw}}$ 
in the presence of a semi-honest adversary 
that can corrupt one party.
}
\begin{proof}
We formally provide the definition 
and prove the security of the two idea functionalities 
in Lemma~\ref{thm:matmul_fw} and~\ref{thm:matmul_bw}, 
respectively. 
Putting them together, we complete the proof for 
Theorem~\ref{thm:matmul_source}.
\end{proof}

\begin{shadedbox}
\centerline{Forward propagation of 
the \texttt{MatMul} source layer 
$\mathcal{F}_\texttt{MatMulFw}$}

\noindent\textbf{Inputs:}
\begin{itemize}[leftmargin=*,label=$\triangleright$]
\item \host inputs features $\xa$, secretly shared and/or encrypted models $\ua, \encb{\va}$, and keys $\ska, \pkb$; 
\item \guest inputs features $\xb$, secretly shared and/or encrypted models $\ub, \enca{\vb}$, and keys $\pka, \skb$.
\end{itemize}

\noindent\textbf{Outputs:}
\begin{itemize}[leftmargin=*,label=$\triangleright$]
\item \host outputs nothing;
\item \guest outputs $Z = \xa\twa + \xb\twb$, where $\twdot = \udot + \vdot$.
\end{itemize}
\end{shadedbox}

\begin{lemma}
\label{thm:matmul_fw}
The protocol $\Pi_{\texttt{MatMulFw}}$ 
in Line 5-8 of Figure~\ref{fig:matmul_routine} 
securely realizes $\mathcal{F}_{\texttt{MatMulFw}}$ 
in the presence of a semi-honest adversary 
that can corrupt one party.
\end{lemma}
\begin{proof}
First, \host only receives one message 
by invoking the $\Pi_{\texttt{HE2SS}}$ protocol 
(Line 6), 
whilst the other values are computed locally. 
Therefore, the view of \host can be perfectly simulated 
by simulating the $\mathcal{F}_{\texttt{HE2SS}}$ 
functionality as discussed in Lemma~\ref{thm:he2ss}. 

Second, \guest receives two messages in the protocol 
(Line 5 and Line 8). 
Similarly, the first message can be simulated 
by simulating the $\mathcal{F}_{\texttt{HE2SS}}$ functionality. 
Denote the simulated versions of $\xa\va - \eps_A$ 
and $\eps_B$ 
are $(\xa\va - \eps_A)^*$ and $\eps_B^*$, respectively. 
Then, we can simulate the second message, 
i.e. $Z_A^\prime$, by computing 
$Z_A^{\prime*} = Z - Z_B^{\prime*}$, 
where $Z_B^{\prime*} = \xb\ub + \eps_B^* + 
	(\xa\va - \eps_A)^*$. 
Since both $Z_A^\prime$ (or $Z_B^\prime$) 
and $Z_A^{\prime*}$ (or $Z_B^{\prime*}$) 
represent one piece of sharing of the output $Z$, 
they have the same probability distribution. 
As a result, we perfectly simulate the view of \guest. 
\end{proof}

\begin{shadedbox}
\centerline{Backward propagation of 
the \texttt{MatMul} source layer 
$\mathcal{F}_\texttt{MatMulBw}$}

\noindent\textbf{Inputs:}
\begin{itemize}[leftmargin=*,label=$\triangleright$]
\item \host inputs features $\xa$, secretly shared and/or encrypted models $\ua, \encb{\va}$, and keys $\ska, \pkb$; 
\item \guest inputs features $\xb$, derivatives $\nabla Z$, secretly shared and/or encrypted models $\ub, \enca{\vb}$, and keys $\pka, \skb$.
\end{itemize}

\noindent\textbf{Outputs:}
\begin{itemize}[leftmargin=*,label=$\triangleright$]
\item \host outputs $\phi$;
\item \guest outputs $\nabla\twa - \phi, \nabla\twb$, where $\nabla\twdot = \xdot^T \nabla Z$.
\end{itemize}
\end{shadedbox}

\begin{lemma}
\label{thm:matmul_bw}
The protocol $\Pi_{\texttt{MatMulBw}}$ 
in Line 9-12 of Figure~\ref{fig:matmul_routine} 
securely realizes $\mathcal{F}_{\texttt{MatMulBw}}$ 
in the presence of a semi-honest adversary 
that can corrupt one party.
\end{lemma}
\begin{proof}
First, \host receives two messages, which are 
$\encb{\nabla Z}$ (Line 9) and 
$\encb{\nabla\twa - \phi}$ (Line 12). 
We construct a simulator that 
randomly picks plaintexts 
$\nabla Z^*, (\nabla\twa - \phi)^*$ 
and encrypts them using $\pkb$ 
to obtain $\encb{\nabla Z^*}, \enc{(\nabla\twa - \phi)^*}$. 
Since $\nabla Z$ is the input and 
$\nabla\twa - \phi$ is one piece of random secret, 
$\nabla Z^*$ (or $(\nabla\twa - \phi)^*$) 
and $\nabla Z$ (or $\nabla\twa - \phi$) 
share the same probability distribution. 
Furthermore, without \guest's secret key $\skb$, 
the ciphertexts 
$\encb{\nabla Z}$ (or $\enc{\nabla\twa - \phi}$) and 
$\encb{\nabla Z^*}$ (or $\enc{(\nabla\twa - \phi)^*}$) 
are computationally indistinguishable 
from the perspective of \host. 
The other values in the view of \host 
are computed locally. 
For instance, the simulator can simulate 
$\encb{\twa}$ by computing 
$\encb{\twa^*} = \xa^T\encb{\nabla Z^*}$. 
Again, they are computationally indistinguishable 
from the perspective of \host. 
Consequently, the view of \host can be simulated perfectly. 

Second, \guest only receives one messages 
by invoking the $\Pi_{\texttt{HE2SS}}$ protocol 
(Line 10), 
whilst the other values are computed locally. 
Therefore, the view of \guest can be perfectly simulated 
by simulating the $\mathcal{F}_{\texttt{HE2SS}}$ 
functionality as discussed in Lemma~\ref{thm:he2ss}. 
\end{proof}

\textsc{Theorem~\ref{thm:matmul_top}.}
{\em
Given $Z, \nabla Z$ in the \texttt{MatMul} source layer, 
there are infinite possible values for $\xa, \twa, \xa\twa$. 
}
\begin{proof}
First, we consider the linear equation $Z = Z_A + Z_B$ 
where $Z_A = \xa\twa, Z_B = \xb\twb$. 
With only $Z \in \mathbb{R}^{BS \times OUT}$ 
being known to \guest, 
there are $BS \times OUT$ known values in total. 
However, since both $Z_A, Z_B$ are unknown, 
which make up to be $2 \times BS \times OUT$ variables 
in the equation. 
As a result, there must be infinite possible solution 
to the equation, which means 
there are infinite possible values for $Z_A = \xa\twa$. 
Then, for any possible $\xa\twa$, 
given an arbitrary invertible matrix 
$M \in \mathbb{R}^{IN_A \times IN_A}$, 
we have $(\xa M^{-1}) (M \twa) = \xa\twa$. 
Consequently, there are infinite possible values 
for $\xa, \twa$, since both $\xa\twa$ and $M$ are arbitrary. 

For $\nabla Z$, since it is computed locally on \guest, 
no extra information related to $\xa, \twa, \xa\twa$ 
are provided. Thus, there are still infinite possible values 
for $\xa, \twa, \xa\twa$. 
\end{proof}

\subsection{Proofs for Section~\ref{sec:embed_matmul_security}}

\textsc{Theorem~\ref{thm:embed_matmul_source}.}
{\em
The protocol of \texttt{Embed-MatMul} source layer 
securely realizes the ideal functionalities 
$\mathcal{F}_{\texttt{EmbedMatMulFw}}$ 
and $\mathcal{F}_{\texttt{EmbedMatMulBw}}$ 
in the presence of a semi-honest adversary 
that can corrupt one party.
}
\begin{proof}
We formally provide the definition 
and prove the security of the two idea functionalities 
in Lemma~\ref{thm:embed_matmul_fw} 
and~\ref{thm:embed_matmul_bw}, 
respectively. 
Putting them together, we complete the proof for 
Theorem~\ref{thm:embed_matmul_source}.
\end{proof}

\begin{shadedbox}
\centerline{Forward propagation of 
the \texttt{Embed-MatMul} source layer 
$\mathcal{F}_\texttt{EmbedMatMulFw}$}

\noindent\textbf{Inputs:}
\begin{itemize}[leftmargin=*,label=$\triangleright$]
\item \host inputs features $\xa$, secretly shared and/or encrypted models $\sa, \encb{\ta}, \tb, \ua, \encb{\va}, \encb{\ub}, \vb$, and keys $\ska, \pkb$; 
\item \guest inputs features $\xb$, secretly shared and/or encrypted models $\sb, \enca{\tb}, \ta, \ub, \enca{\vb}, \enca{\ua}, \va$, and keys $\pka, \skb$. 
\end{itemize}

\noindent\textbf{Outputs:}
\begin{itemize}[leftmargin=*,label=$\triangleright$]
\item \host outputs nothing;
\item \guest outputs $Z = \ea\twa + \eb\twb$, where $\edot = \lk(\qdot, \xdot), \qdot = \sdot + \tdot, \twdot = \udot + \vdot$.
\end{itemize}
\end{shadedbox}

\begin{lemma}
\label{thm:embed_matmul_fw}
The protocol $\Pi_{\texttt{EmbedMatMulFw}}$ 
in Line 5-11 of Figure~\ref{fig:embed_matmul_routine} 
securely realizes $\mathcal{F}_{\texttt{EmbedMatMulFw}}$ 
in the presence of a semi-honest adversary 
that can corrupt one party.
\end{lemma}
\begin{proof}
First, \host receives three messages in total. 
In Line 6, \host receives one message 
by the $\Pi_{\texttt{He2SS}}$ protocol. 
In Line 7-8, \host invokes 
the \texttt{MatMulFw} routine twice, 
and each receives one message 
by the $\Pi_{\texttt{HE2SS}}$ protocol, 
as discussed in our proof of Lemma~\ref{thm:matmul_fw}. 
Whilst all other values are computed locally. 
Obviously, the view of \host can be perfectly simulated 
by simulating the $\mathcal{F}_{\texttt{HE2SS}}$ 
functionality. 

Second, \guest receives four messages in the protocol. 
Symmetrical as \host, the first three messages are received 
via the $\Pi_{\texttt{He2SS}}$ protocol. 
We can simulate these three messages 
by simulating the $\mathcal{F}_{\texttt{HE2SS}}$ 
functionality. 
The fourth message is received in Line 11. 
Denote the simulated versions of 
$Z_{1,B}^\prime$ and $Z_{2,B}^\prime$ 
are $Z_{1,B}^{\prime*}$ and $Z_{2,B}^{\prime*}$, 
respectively. 
Then, we can simulate the fourth message, 
i.e. $Z_A^\prime$, by computing 
$Z_A^{\prime*} = Z - Z_B^{\prime*}$, 
where $Z_B^{\prime*} = Z_{1,B}^{\prime*} + Z_{2,B}^{\prime*}$. 
Since both $Z_A^\prime$ (or $Z_B^\prime$) and 
$Z_A^{\prime*}$ (or $Z_B^{\prime*}$) 
represent one piece of sharing of the output $Z$, 
they have the same probability distribution. 
As a result, we perfectly simulate the view of \guest. 
\end{proof}

\begin{shadedbox}
\centerline{Backward propagation of 
the \texttt{Embed-MatMul} source layer 
$\mathcal{F}_\texttt{EmbedMatMulBw}$}

\noindent\textbf{Inputs:}
\begin{itemize}[leftmargin=*,label=$\triangleright$]
\item \host inputs features $\xa$, secretly shared activations $\psi_A, \eb - \psi_B$, secretly shared and/or encrypted models $\sa, \encb{\ta}, \tb, \ua, \encb{\va}, \\\encb{\ub}, \vb$, and keys $\ska, \pkb$; 
\item \guest inputs features $\xb$, derivatives $\nabla Z$, secretly shared activations $\ea - \psi_A, \psi_B$, secretly shared and/or encrypted models $\sb, \enca{\tb}, \ta, \ub, \enca{\vb}, \enca{\ua}, \va$, and keys $\pka, \skb$. 
\end{itemize}

\noindent\textbf{Outputs:}
\begin{itemize}[leftmargin=*,label=$\triangleright$]
\item \host outputs $\phi, \xi, \rho_A, \nabla\qa - \rho_B$;
\item \guest outputs $\nabla\twa - \phi, \nabla\twb - \xi, \nabla\qa - \rho_A, \rho_B$, where $\nabla\twdot = \edot^T \nabla Z, \nabla\edot = \nabla Z \twdot^T, \nabla\qdot = \dlk(\nabla\edot, \xdot)$.
\end{itemize}
\end{shadedbox}

\begin{lemma}
\label{thm:embed_matmul_bw}
The protocol $\Pi_{\texttt{EmbedMatMulBw}}$ 
in Line 12-26 of Figure~\ref{fig:embed_matmul_routine} 
securely realizes $\mathcal{F}_{\texttt{EmbedMatMulBw}}$ 
in the presence of a semi-honest adversary 
that can corrupt one party.
\end{lemma}
\begin{proof}
First, \host receives six messages, which are 
$\encb{\nabla Z}$, $\encb{\nabla Z \va^T}$ (Line 12), 
$\encb{\nabla\twa - \phi}$ (Line 18), 
$\encb{\nabla\twb - \xi}$ (Line 19), 
$\nabla\qb - \rho_B$ (Line 23), 
and $\encb{\nabla\qa - \rho_A}$ (Line 25), 
respectively. 
For $\nabla\qb - \rho_B$, which is received 
from the $\Pi_{\texttt{HE2SS}}$ protocol, 
we can simulate this message by simulating 
the $\mathcal{F}_\texttt{HE2SS}$ functionality 
as discussed in Lemma~\ref{thm:he2ss}. 
For the other five messages, 
we construct a simulator that 
randomly picks plaintexts 
$\nabla Z^*$, $(\nabla Z \va^T)^*$
$(\nabla\twa - \phi)^*$, 
$(\nabla\twb - \xi)^*$, $(\nabla\qa - \rho_A)^*$ 
and encrypts them using $\pkb$ to obtain 
$\encb{\nabla Z^*}$, $\encb{(\nabla Z \va^T)^*}$, 
$\encb{(\nabla\twa - \phi)^*}$, 
$\encb{(\nabla\twb - \xi)^*}$, 
$\encb{(\nabla\qa - \rho_A)^*}$. 
Since $\nabla Z$ is the input and 
$\va$, $\nabla\twa - \phi$, $\nabla\twb - \xi$, 
$\nabla\qa - \rho_A$ 
are random secrets, 
$\nabla Z^*$ and $\nabla Z$ 
(or $(\nabla Z \va^T)^*$ and $\nabla Z \va^T$, 
or $(\nabla\twa - \phi)^*$ and $\nabla\twa - \phi$, 
or $(\nabla\twb - \xi)^*$ and $\nabla\twb - \xi$, 
or $(\nabla\qa - \rho_A)^*$ and $\nabla\qa - \rho_A$)
share the same probability distribution. 
Furthermore, without \guest's secret key $\skb$, 
the ciphertexts 
$\encb{\nabla Z^*}$ and $\encb{\nabla Z}$ 
(or $\encb{(\nabla\twa - \phi)^*}$ and $\encb{\nabla\twa - \phi}$, 
or $\encb{(\nabla\twb - \xi)^*}$ and $\encb{\nabla\twb - \xi}$, 
or $\encb{(\nabla\qa - \rho_A)^*}$ and $\encb{\nabla\qa - \rho_A}$)
are computationally indistinguishable 
from the perspective of \host. 
The other values in the view of \host 
are computed locally. 
For instance, the simulator can simulate 
$\encb{\nabla\ea}$ by computing 
$\encb{\nabla\ea^*} = \encb{\nabla Z^*} \ua^T + 
\encb{(\nabla Z \va^T)^*}$. 
Again, they are computationally indistinguishable 
from the perspective of \host. 
Consequently, the view of \host can be simulated perfectly. 

For \guest, there are three messages received by invoking 
the $\Pi_{\texttt{HE2SS}}$ protocol 
(Line 13, 15, and 22), 
and three messages that are ciphertexts of secrets 
(Line 17, 20, 26). 
For the first three messages, 
we can simulate them by simulating 
the $\mathcal{F}_\texttt{HE2SS}$ functionality. 
For the last three messages, 
we simulate them in a similar way 
as simulating those of \host, 
i.e., by randomly picking ciphertexts 
and encrypting them using $\pka$. 
Again, since the messages are ciphertexts of random secrets, 
the simulation shares the same probability distribution 
as the original view 
and they are computationally indistinguishable 
from the perspective of \guest. 
As a result, the simulator perfectly simulates 
the view of \guest. 
\end{proof}

\textsc{Theorem~\ref{thm:matmul_top}.}
{\em
Given $Z, \nabla Z$ in the \texttt{Embed-MatMul} source layer, 
there are infinite possible values for 
$\xa, \qa, \twa, \ea, \ea\twa$. 
}
\begin{proof}
First, we consider the linear equation $Z = Z_A + Z_B$ 
where $Z_A = \ea\twa, Z_B = \eb\twb$. 
With only $Z \in \mathbb{R}^{BS \times OUT}$ 
being known to \guest, 
there are $BS \times OUT$ known values in total. 
However, since both $Z_A, Z_B$ are unknown, 
which make up to be $2 \times BS \times OUT$ variables 
in the equation. 
As a result, there must be infinite possible solution 
to the equation, which means 
there are infinite possible values for $Z_A = \xa\ea$. 
Then, for any possible $\xa\ea$, 
given an arbitrary invertible matrix 
$M \in \mathbb{R}^{IN_A \times IN_A}$, 
we have $(\ea M^{-1}) (M \twa) = \ea\twa$. 
Consequently, there are infinite possible values 
for $\ea, \twa$, since both $\ea\twa$ and $M$ are arbitrary. 
Similarly, for any possible $\ea$, 
there are also infinite possible values 
for $\xa, \qa$. 

For $\nabla Z$, since it is computed locally on \guest, 
no extra information related to 
$\xa, \qa, \twa, \ea, \ea\twa$ 
are provided. Thus, there are still infinite possible values 
for $\xa, \qa, \twa, \ea, \ea\twa$. 
\end{proof}

\subsection{Proofs for Common Sub-Routines}

Algorithm~\ref{alg:he2ss} and Algorthm~\ref{alg:ss2he} 
are two commonly used sub-routines in our work. 
In this subsection, we prove that the algorithm protocols 
securely realize the ideal functionalities 
as defined below. 

\medskip

\begin{shadedbox}
\centerline{Transformation from HE variables to SS variables 
$\mathcal{F}_{\texttt{HE2SS}}$}

\noindent\textbf{Inputs:}
\begin{itemize}[leftmargin=*,label=$\triangleright$]
\item \party $\mydot$ inputs a ciphertext $\enc{v}_\otherdot$ and the other party's public key $\pkotherdot$; 
\item \party $\otherdot$ inputs the secret key $\skotherdot$. 
\end{itemize}

\noindent\textbf{Outputs:}
\begin{itemize}[leftmargin=*,label=$\triangleright$]
\item \party $\mydot$ outputs a random value $\phi$;
\item \party $\otherdot$ outputs $v - \phi$.
\end{itemize}
\end{shadedbox}

\begin{lemma}
\label{thm:he2ss}
The protocol $\Pi_{\texttt{HE2SS}}$ 
in Algorithm~\ref{alg:he2ss} 
securely realizes $\mathcal{F}_{\texttt{HE2SS}}$ 
in the presence of a semi-honest adversary 
that can corrupt one party.
\end{lemma}
\begin{proof}
Since \party $\mydot$ receives no messages 
from \party $\otherdot$, 
it is easy to know that 
our protocol is secure against the corruption 
of \party $\mydot$. 
Hence, we only need to discuss how to 
simulate the view of \party $\otherdot$. 
To do so, the simulator encrypts $v-\phi$ with $\pkb$ 
to obtain $\enc{v - \phi}^*_\otherdot$. 
Here we differentiate $\enc{v - \phi}^*\otherdot$ 
and the $\enc{v - \phi}_\otherdot$ 
in Algorithm~\ref{alg:he2ss} 
since they are from different times of encryption. 
Obviously, they share the same probability distribution 
and are computationally indistinguishable 
from the perspective of \party $\otherdot$. 
As a result, it perfectly simulates the view 
of \party $\otherdot$. 
\end{proof}

\begin{shadedbox}
\centerline{Transformation from SS variables to HE variables
$\mathcal{F}_{\texttt{SS2HE}}$}

\noindent\textbf{Inputs:}
\begin{itemize}[leftmargin=*,label=$\triangleright$]
\item \party $\mydot$ inputs a piece of SS $v_\mydot$ and keys $\pkdot, \pkotherdot$; 
\item \party $\otherdot$ inputs a piece of SS $v_\otherdot$ and keys $\pkotherdot, \pkdot$. 
\end{itemize}

\noindent\textbf{Outputs:}
\begin{itemize}[leftmargin=*,label=$\triangleright$]
\item \party $\mydot$ outputs a ciphertext $\enc{v}_\otherdot$; 
\item \party $\otherdot$ outputs a ciphertext $\enc{v}_\mydot$, where $v = v_\mydot + v_\otherdot$.
\end{itemize}
\end{shadedbox}

\begin{lemma}
\label{thm:ss2he}
The protocol $\Pi_{\texttt{SS2HE}}$ 
in Algorithm~\ref{alg:ss2he} 
securely realizes $\mathcal{F}_{\texttt{SS2HE}}$ 
in the presence of a semi-honest adversary 
that can corrupt one party.
\end{lemma}
\begin{proof}
Since both parties have symmetric routines, 
we only describe how to simulate the view 
of \party $\mydot$. 
The only messages received in the protocol is 
$\enc{v_\otherdot}_\otherdot$. 
To simulate it, the simulator computes 
$\enc{v_\otherdot}_\otherdot^* = 
\enc{v}_\otherdot - v_\mydot$. 
Here we differentiate $\enc{v_\otherdot}_\otherdot^*$ 
and the $\enc{v_\otherdot}_\otherdot$ 
in Algorithm~\ref{alg:ss2he} 
since they are from different times of encryption. 
Obviously, they have the same probability distribution 
since $v = v_\mydot + v_\otherdot$, 
and they are computationally indistinguishable 
from the perspective of \party $\mydot$. 
As a result, the view of \party $\mydot$ 
can be perfectly simulated.
\end{proof}

\section{Extension to Federated Top Models}
\label{sec:ext_fed_top}

Although the top models are usually non-federated 
to address efficiency in practice, 
we can also use a federated top model 
so that \guest cannot get access to $Z$ or $\nabla Z$, 
strengthening the security guarantees. 
In this section, we discuss how to adapt 
our work to the federated top model. 
Without loss of generality, 
we assume the top model utilizes 
the SS technique (e.g., SecureML).

First, for the forward propagation, 
our source layers should 
output an SS variable $\ss{Z_A^\prime, Z_B^\prime}$ 
satisfying $Z_A^\prime + Z_B^\prime = Z$. 
In fact, this can be easily done with 
our algorithm protocols, as described below. 
\begin{itemize}[leftmargin=*,label=$\triangleright$]
\item 
For the \texttt{MatMul} source layer, we have the fact that  
$Z_\mydot^\prime = \xdot\udot + \eps_\mydot 
+ \xodot\vodot - \eps_\otherdot$, 
where $\eps_\mydot, \eps_\otherdot$ are generated by 
different parties. 
Hence, $\ss{Z_A^\prime, Z_B^\prime}$ 
forms an SS variable of $Z$ logically. 
As shown in Line 1 of Figure~\ref{fig:matmul_fed_top}, 
\host and \guest return $Z_A^\prime$ and $Z_B^\prime$, 
respectively. 
By doing so, we can easily adapt 
the \texttt{MatMul} source layer 
to the federated top model.

\item
Similarly, for the \texttt{Embed-MatMul} source layer, 
we can also obtain the SS variable 
$\ss{Z_A^\prime, Z_B^\prime}$ 
satisfying $Z = Z_A^\prime, Z_B^\prime$ 
via Line 5-10 of Figure~\ref{fig:embed_matmul_routine}. 
Thus, we directly output them to the federated top model. 
\end{itemize}

Second, for the backward propagation, 
\guest no longer gets access to $\nabla Z$. 
Instead, both parties take as input an SS variable 
$\ss{\eps, \nabla Z - \eps}$. 
For the \texttt{MatMul} source layer, 
the major difference is that \guest cannot compute 
model gradients $\nabla\twb = \xb^T \nabla Z$ alone. 
To tackle this problem, we turn the SS variable 
into an HE variable $\enca{\nabla Z}$ for \guest 
via Algorithm~\ref{alg:ss2he} 
(Line 3 of Figure~\ref{fig:matmul_fed_top}). 
Next, \guest proceeds to compute the encrypted model gradients 
$\enca{\nabla\twb} = \xb^T \enca{\nabla Z}$, 
which are then transformed into an SS variable 
$\ss{\nabla\twb - \phi_B, \phi_B}$ 
(Line 4-5 of Figure~\ref{fig:matmul_fed_top}). 
Eventually, the update will be performed 
based on the secretly shared mode gradients. 
For the \texttt{Embed-MatMul} source layer, 
the routines are more complex. 
We refer interested readers to 
Figure~\ref{fig:embed_matmul_fed_top} 
for more details. 

To summarized, the federated source layers 
proposed in this work can also be integrated with 
federated top models.

\begin{algorithm}[!t]
\caption{{The procedure to transform an SS variable 
$\langle v_\mydot, v_\otherdot \rangle$ 
into an HE variable $\enc{v}$ 
where $v = v_\mydot + v_\otherdot$.}}
\label{alg:ss2he}
\DontPrintSemicolon
\SetKwFunction{SStoHE}{SS2HE}
\SetKwProg{Fn}{Function}{:}{}

\Fn{\SStoHE{$v_\mydot$}}
{
	Enc and Send the SS piece of this party $v_\mydot$ using $\pkdot$\;
	Recv the encrypted SS piece of the other party $\enc{v_\otherdot}_\otherdot$\;
	\Return $\enc{v}_\otherdot = \enc{v_\otherdot}_\otherdot + v_\mydot$
}
\end{algorithm}

\begin{figure}[!t]
\centering
\includegraphics[width=3.4in]{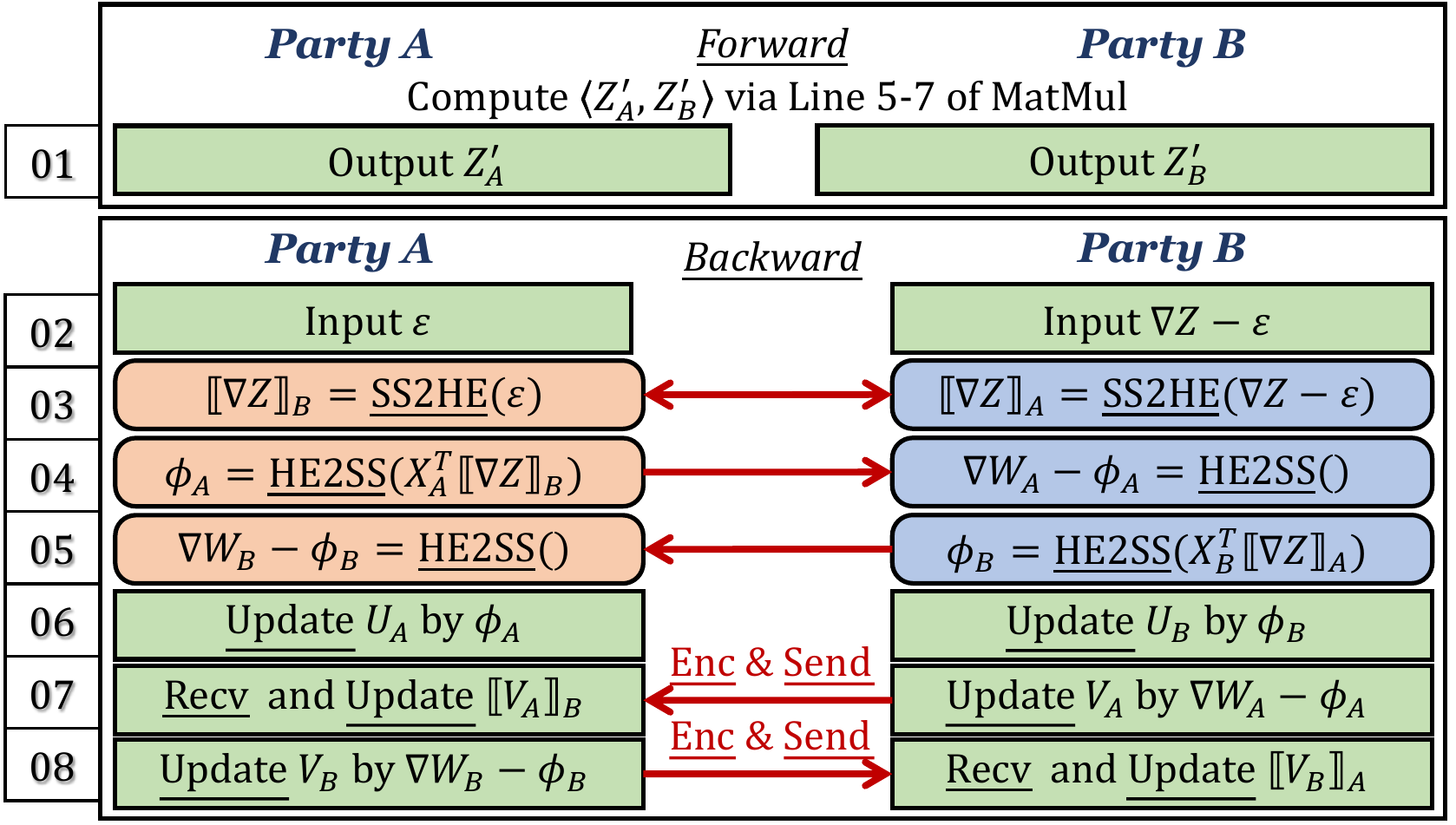}
\caption{{
The \texttt{MatMul} source layer followed by an SS-based top model.
All cross-party transmission (red arrows) 
are protected by HE or SS.
}}
\label{fig:matmul_fed_top}
\end{figure}

\begin{figure}[!t]
\centering
\includegraphics[width=3.4in]{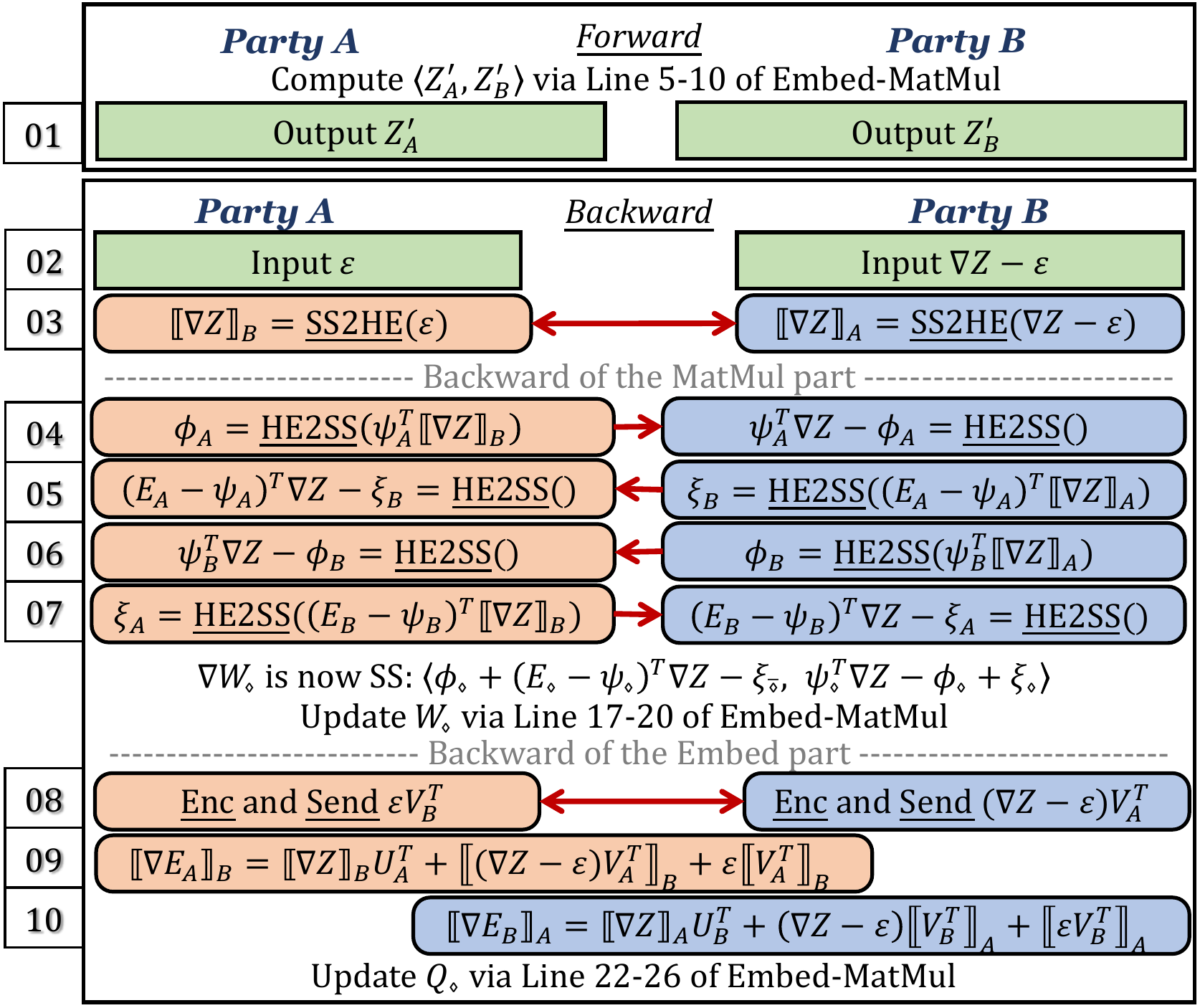}
\caption{{
The \texttt{Embed-Matmul} source layer followed by an SS-based top model.
All cross-party transmission (red arrows) 
are protected by HE or SS.
}}
\label{fig:embed_matmul_fed_top}
\end{figure}

\subsection{Security Analysis}
In this subsection, 
we provide the security analysis of \alg 
when the top model is also a federated module. 
Without loss of generality, 
we assume the model weights of the federated top model 
are also secretly shared and/or encrypted. 
Furthermore, since our work focuses on 
the security guarantees of federated source layers, 
we assume the federated top model 
securely realizes the following ideal functionality 
for simplicity.
\medskip

\begin{shadedbox}
\centerline{Forward and backward propagation of 
SS-based top model $\mathcal{F}_\texttt{TopSS}$}

\noindent\textbf{Inputs:}
\begin{itemize}[leftmargin=*,label=$\triangleright$]
\item \host inputs secretly shared activations $Z_A^\prime$, secretly shared and/or encrypted models of the top model,  and keys $\ska, \pkb$; 
\item \guest inputs inputs secretly shared activations $Z_B^\prime$, labels $y$, secretly shared and/or encrypted models of the top model, and keys $\pka, \skb$.
\end{itemize}

\noindent\textbf{Outputs:}
\begin{itemize}[leftmargin=*,label=$\triangleright$]
\item \host outputs $\eps$;
\item \guest outputs $\nabla Z - \eps$, where $\nabla Z$ denotes the backward derivatives of $Z = Z_A^\prime + Z_B^\prime$.
\end{itemize}
\end{shadedbox}

\medskip

Then, we identify the ideal functionality 
of ML training of \alg with an SS-based top model, 
and provide the security guarantees 
in Theorem~\ref{thm:fed_source_top_security}. 

\medskip

\begin{shadedbox}
\centerline{ML Training of \alg with an SS-based top model 
$\mathcal{F}_\texttt{ML}$}

\noindent\textbf{Inputs:}
\begin{itemize}[leftmargin=*,label=$\triangleright$]
\item \host inputs features $\xa$ and keys $\ska, \pkb$; 
\item \guest inputs features $\xb$, labels $y$, and keys $\pka, \skb$.
\end{itemize}

\noindent\textbf{Outputs:}
\begin{itemize}[leftmargin=*,label=$\triangleright$]
\item Both parties output the trained federated source layers and \guest further outputs the trained federated top model (model weights are secretly shared or encrypted).
\end{itemize}
\end{shadedbox}

\begin{theorem}
\label{thm:fed_source_top_security}
Assume the federated top model $\Pi_\texttt{TopSS}$ 
securely realizes $\mathcal{F}_\texttt{TopSS}$ 
in the presence of a semi-honest adversary 
that can corrupt one party. 
The training of \alg using the \texttt{MatMul} 
and \texttt{Embed-MatMul} federated source layers 
securely realizes $\mathcal{F}_\texttt{ML}$ 
in the presence of a semi-honest adversary 
that can corrupt one party. 
\end{theorem}
\begin{proof}
Since all values are secretly shared and/or encrypted 
during the training process, 
it is actually a hybrid interaction 
of the $\Pi_{\texttt{MatMulSSFw}}$, 
$\Pi_{\texttt{MatMulSSBw}}$, 
$\Pi_{\texttt{Embed-MatMulSSFw}}$, 
$\Pi_{\texttt{Embed-MatMulSSBw}}$, 
and $\Pi_\texttt{TopSS}$ protocols. 
Therefore, simulation can be done 
using the hybrid argument. 
Combing Lemma~\ref{thm:matmul_ss_fw}, ~\ref{thm:matmul_ss_bw}, ~\ref{thm:embed_matmul_ss_fw}, ~\ref{thm:embed_matmul_ss_bw}, 
and the assumption on the federated top model, 
we prove Theorem~\ref{thm:fed_source_top_security}.
\end{proof}

\begin{shadedbox}
\centerline{Forward propagation of 
the \texttt{MatMul} source layer}

\centerline{followed by an SS-based top model  
$\mathcal{F}_\texttt{MatMulSSFw}$}

\noindent\textbf{Inputs:}
\begin{itemize}[leftmargin=*,label=$\triangleright$]
\item \host inputs features $\xa$, secretly shared and/or encrypted models $\ua, \encb{\va}$, and keys $\ska, \pkb$; 
\item \guest inputs features $\xb$, secretly shared and/or encrypted models $\ub, \enca{\vb}$, and keys $\pka, \skb$.
\end{itemize}

\noindent\textbf{Outputs:}
\begin{itemize}[leftmargin=*,label=$\triangleright$]
\item \host outputs $Z_A^\prime$;
\item \guest outputs $Z_B^\prime$, where $Z_A^\prime + Z_B^\prime = \xa\twa + \xb\twb, \twdot = \udot + \vdot$.
\end{itemize}
\end{shadedbox}

\begin{lemma}
\label{thm:matmul_ss_fw}
The protocol $\Pi_{\texttt{MatMulSSFw}}$ 
in Line 1 of Figure~\ref{fig:matmul_fed_top} 
securely realizes $\mathcal{F}_{\texttt{MatMulSSFw}}$ 
in the presence of a semi-honest adversary 
that can corrupt one party.
\end{lemma}
\begin{proof}
This lemma can be proved similarly 
as Lemma~\ref{thm:matmul_fw}. 
Thus, we would like to refer readers to the proof 
of Lemma~\ref{thm:matmul_fw}. 
\end{proof}

\begin{shadedbox}
\centerline{Backward propagation of 
the \texttt{MatMul} source layer}

\centerline{followed by an SS-based top model  
$\mathcal{F}_\texttt{MatMulSSBw}$}

\noindent\textbf{Inputs:}
\begin{itemize}[leftmargin=*,label=$\triangleright$]
\item \host inputs features $\xa$, secretly shared derivatives $\eps$, secretly shared and/or encrypted models $\ua, \encb{\va}$, and keys $\ska, \pkb$; 
\item \guest inputs features $\xb$, secretly shared derivatives $\nabla Z - \eps$, secretly shared and/or encrypted models $\ub, \enca{\vb}$, and keys $\pka, \skb$.
\end{itemize}

\noindent\textbf{Outputs:}
\begin{itemize}[leftmargin=*,label=$\triangleright$]
\item \host outputs $\phi_A, \nabla\twb - \phi_B$;
\item \guest outputs $\nabla\twa - \phi_A, \phi_B$, where $\nabla\twdot = \xdot^T \nabla Z$.
\end{itemize}
\end{shadedbox}

\begin{lemma}
\label{thm:matmul_ss_bw}
The protocol $\Pi_{\texttt{MatMulSSBw}}$ 
in Line 2-8 of Figure~\ref{fig:matmul_fed_top} 
securely realizes $\mathcal{F}_{\texttt{MatMulSSBw}}$ 
in the presence of a semi-honest adversary 
that can corrupt one party.
\end{lemma}
\begin{proof}
Since both parties have symmetric routines, 
we only describe how to simulate the view 
of \host. 
There are three messages received in the protocol. 
For the first two messages, i.e., 
$\encb{\nabla Z}, \nabla\twb - \phi_B$ (Line 3 and 5), 
they can be simulated by simulating 
the $\Pi_{\texttt{SS2HE}}$ and $\Pi_{\texttt{HE2SS}}$ 
protocols, respectively. 
For the third message, i.e., 
$\encb{\nabla\twa - \phi_A}$ (Line 7), 
we construct a simulator that 
randomly picks plaintexts 
$(\nabla\twa - \phi_A)^*$ 
and encrypts them using $\pkb$ 
to obtain $\enc{(\nabla\twa - \phi_A)^*}$. 
Since $\nabla\twa - \phi_A$ is one piece of random secret, 
$(\nabla\twa - \phi_A)^*$ and $\nabla\twa - \phi_A$ 
share the same probability distribution. 
Furthermore, without \guest's secret key $\skb$, 
$\enc{\nabla\twa - \phi}$ and $\enc{(\nabla\twa - \phi)^*}$ 
are computationally indistinguishable 
from the perspective of \host. 
The other values in the view of \host 
are computed locally. 
For instance, denoting $\encb{\nabla Z^*}$ as 
the simulated version of $\encb{\nabla Z}$, 
the simulator can simulate 
$\encb{\twa}$ by computing 
$\encb{\twa^*} = \xa^T\encb{\nabla Z^*}$. 
Again, they are computationally indistinguishable 
from the perspective of \host. 
Consequently, the view of \host can be simulated perfectly. 
\end{proof}

\begin{shadedbox}
\centerline{Forward propagation of 
the \texttt{Embed-MatMul} source layer}

\centerline{followed by an SS-based top model  
$\mathcal{F}_\texttt{EmbedMatMulSSFw}$}

\noindent\textbf{Inputs:}
\begin{itemize}[leftmargin=*,label=$\triangleright$]
\item \host inputs features $\xa$, secretly shared and/or encrypted models $\sa, \encb{\ta}, \tb, \ua, \encb{\va}, \encb{\ub}, \vb$, and keys $\ska, \pkb$; 
\item \guest inputs features $\xb$, secretly shared and/or encrypted models $\sb, \enca{\tb}, \ta, \ub, \enca{\vb}, \enca{\ua}, \va$, and keys $\pka, \skb$. 
\end{itemize}

\noindent\textbf{Outputs:}
\begin{itemize}[leftmargin=*,label=$\triangleright$]
\item \host outputs $Z_A^\prime$;
\item \guest outputs $Z_B^\prime$, where $Z_A^\prime + Z_B^\prime = \ea\twa + \eb\twb, \edot = \lk(\qdot, \xdot), \qdot = \sdot + \tdot, \twdot = \udot + \vdot$.
\end{itemize}
\end{shadedbox}

\begin{lemma}
\label{thm:embed_matmul_ss_fw}
The protocol $\Pi_{\texttt{EmbedMatMulSSFw}}$ 
in Line 1 of Figure~\ref{fig:embed_matmul_fed_top} 
securely realizes $\mathcal{F}_{\texttt{EmbedMatMulSSFw}}$ 
in the presence of a semi-honest adversary 
that can corrupt one party.
\end{lemma}
\begin{proof}
This lemma can be proved similarly 
as Lemma~\ref{thm:embed_matmul_fw}. 
Thus, we would like to refer readers to the proof 
of Lemma~\ref{thm:embed_matmul_fw}. 
\end{proof}

\begin{shadedbox}
\centerline{Backward propagation of 
the \texttt{Embed-MatMul} source layer}

\centerline{followed by an SS-based top model  
$\mathcal{F}_\texttt{EmbedMatMulSSBw}$}

\noindent\textbf{Inputs:}
\begin{itemize}[leftmargin=*,label=$\triangleright$]
\item \host inputs features $\xa$, secretly shared derivatives $\eps$, secretly shared activations $\psi_A, \eb - \psi_B$, secretly shared and/or encrypted models $\sa, \encb{\ta}, \tb, \ua, \encb{\va}, \encb{\ub}, \vb$, and keys $\ska, \pkb$; 
\item \guest inputs features $\xb$, secretly shared derivatives $\nabla Z - \eps$, secretly shared activations $\ea - \psi_A, \psi_B$, secretly shared and/or encrypted models $\sb, \enca{\tb}, \ta, \ub, \enca{\vb}, \enca{\ua}, \va$, and keys $\pka, \skb$. 
\end{itemize}

\noindent\textbf{Outputs:}
\begin{itemize}[leftmargin=*,label=$\triangleright$]
\item \host outputs $\phi, \xi, \rho_A, \nabla\qa - \rho_B$;
\item \guest outputs $\nabla\twa - \phi, \nabla\twb - \xi, \nabla\qa - \rho_A, \rho_B$, where $\nabla\twdot = \edot^T \nabla Z, \nabla\edot = \nabla Z \twdot^T, \nabla\qdot = \dlk(\nabla\edot, \xdot)$.
\end{itemize}
\end{shadedbox}

\begin{lemma}
\label{thm:embed_matmul_ss_bw}
The protocol $\Pi_{\texttt{EmbedMatMulSSBw}}$ 
in Line 2-10 of Figure~\ref{fig:embed_matmul_routine} 
securely realizes $\mathcal{F}_{\texttt{EmbedMatMulSSBw}}$ 
in the presence of a semi-honest adversary 
that can corrupt one party.
\end{lemma}
\begin{proof}
Since both parties have symmetric routines, 
we only describe how to simulate the view of \host. 
There are eight messages received in total. 

First, for the four messages received 
in Line 3, 5, 6, and 8 
of Figure~\ref{fig:embed_matmul_fed_top}, 
they are received by 
invoking the $\Pi_{\texttt{HE2SS}}$ 
and $\Pi_{\texttt{SS2HE}}$ protocols, 
so we can simulate them by simulating the 
$\mathcal{F}_{\texttt{HE2SS}}$ 
and $\mathcal{F}_{\texttt{SS2HE}}$ functionalities. 

The $\Pi_{\texttt{EmbedMatMulSSBw}}$ protocol 
in Figure~\ref{fig:embed_matmul_fed_top} 
also involves parts of 
the $\Pi_{\texttt{EmbedMatMulBw}}$ protocol 
in Figure~\ref{fig:embed_matmul_routine}, 
where four messages are also received 
(Line 18, 19, 23, and 25 
of Figure~\ref{fig:embed_matmul_routine}). 
For the message received in Line 23 
of Figure~\ref{fig:embed_matmul_routine}, 
we can simulate it by simulating the 
$\mathcal{F}_{\texttt{HE2SS}}$ functionalities. 
For the other three messages received in 
Line 18, 19, and 25 of Figure~\ref{fig:embed_matmul_routine}, 
we construct a simulator that 
randomly picks plaintexts 
$(\nabla\twa - \phi)^*, (\nabla\twb - \xi)^*, 
(\nabla\qa - \rho_A)^*$ 
and encrypts them using $\pkb$ to obtain 
$\encb{(\nabla\twa - \phi)^*}$, $\encb{(\nabla\twb - \xi)^*}$, 
$\encb{(\nabla\qa - \rho_A)^*}$. 
Since 
$\nabla\twa - \phi, \nabla\twb - \xi, \nabla\qa - \rho_A$  
are random secrets, 
$(\nabla\twa - \phi)^*$ and $\nabla\twa - \phi$ 
(or $(\nabla\twb - \xi)^*$ and $\nabla\twb - \xi$, 
or $(\nabla\qa - \rho_A)^*$ and $\nabla\qa - \rho_A$) 
share the same probability distribution. 
Furthermore, without \guest's secret key $\skb$, 
the ciphertexts 
$\encb{(\nabla\twa - \phi)^*}$ and $\encb{\nabla\twa - \phi}$ 
(or $\encb{(\nabla\twb - \xi)^*}$ and $\encb{\nabla\twb - \xi}$, 
or $\encb{(\nabla\qa - \rho_A)^*}$ and $\encb{\nabla\qa - \rho_A}$) 
are computationally indistinguishable 
from the perspective of \host. 
Therefore, they can be perfectly simulated. 

The other values in the view of \host 
are computed locally. 
For instance, denoting the simulated version of 
$\encb{\nabla Z}$, $\encb{(\nabla Z - \eps) \va^T}$ 
as $\encb{\nabla Z}^*$, $\encb{(\nabla Z - \eps) \va^T}^*$, 
respectively, 
the simulator can simulate $\encb{\ea}$ by computing 
$\encb{\ea}^* = \encb{\nabla Z}^* \ua^T + 
\encb{(\nabla Z - \eps) \va^T}^* + 
\eps \encb{\va^T}$. 
Again, they are computationally indistinguishable 
from the perspective of \host. 
Consequently, the view of \host can be simulated perfectly. 
\end{proof}

\section{Extension to Multi-Party Learning}
\label{sec:ext_multi}

\newcommand{\hosti}{\textit{Party A(i)}\xspace}

\newcommand{\ai}{{A(i)}}
\newcommand{\bi}{{B(i)}}
\newcommand{\xai}{X_\ai}
\newcommand{\twai}{W_\ai}
\newcommand{\uai}{U_\ai}
\newcommand{\vai}{V_\ai}
\newcommand{\vbi}{V_\bi}
\newcommand{\qai}{Q_\ai}
\newcommand{\sai}{S_\ai}
\newcommand{\tai}{T_\ai}
\newcommand{\tbi}{T_\bi}
\newcommand{\eai}{E_\ai}
\newcommand{\ebi}{E_\bi}

\begin{algorithm}[!t]
\caption{The \texttt{MatMul} source layer under the multi-party setting. We assume there are $M$ \host's.}
\label{alg:matmul_multi}
\DontPrintSemicolon
\SetKwFunction{MatMulInit}{MatMulInit}
\SetKwFunction{MatMulFw}{MatMulFw}
\SetKwFunction{MatMulBw}{MatMulBw}
\SetKwFunction{MultiPartyMatMulInit}{MultiPartyMatMulInit}
\SetKwFunction{MultiPartyMatMulFw}{MultiPartyMatMulFw}
\SetKwFunction{MultiPartyMatMulBw}{MultiPartyMatMulBw}
\SetKwFunction{HEtoSS}{HE2SS}
\SetKwProg{Fn}{Function}{:}{}

\Fn{\MultiPartyMatMulInit{$IN_\ai, IN_B, OUT$}}
{
	\uIf {is \textup{\guest}}
	{
		Initialize $\ub \in \mathbb{R}^{IN_B \times OUT}$\;
		\ForEach{$i \gets 1$ \KwTo $M$}
		{
			Initialize $\vai \in \mathbb{R}^{IN_\ai \times OUT}$\;
			Enc and Send $\vai$; $\;$ 
			Recv $\enc{\vbi}_\ai$\; 
		}
		\Return $\ub, \enc{\vbi}_\ai, \vai$
	}
	\uElse (\tcp*[h]{\hosti})
	{
		Initialize $\uai \in \mathbb{R}^{IN_\ai \times OUT}, 
			\vbi \in \mathbb{R}^{IN_B \times OUT}$\;
		Enc and Send $\vbi$; $\;$
		Recv $\enc{\vai}_B$\;
		\Return $\uai, \enc{\vai}_B, \vbi$
	}
}

\Fn{\MultiPartyMatMulFw{$\xdot, \udot, \enc{\vdot}_\otherdot$}}
{
	\uIf {is \textup{\guest}}
	{
		\ForEach{$i \gets 1$ \KwTo $M$}
		{
			\tcc{Line 5-8 of Figure~\ref{fig:matmul_routine}}
			$Z_i =\ $\MatMulFw{$\xb, \ub / M, \enc{\vbi}_\ai$}\;
		}
		\Return $Z = \sum_i Z_i$
	}
	\uElse (\tcp*[h]{\hosti})
	{
		\tcc{Line 5-8 of Figure~\ref{fig:matmul_routine}}
		\MatMulFw($\xa, \uai, \encb{\vai}$)\;
		\Return null\;
	}
}

\Fn{\MultiPartyMatMulBw{$\xdot, \udot, \vodot$, $\nabla Z$}}
{
	\uIf {is \textup{\guest}}
	{
		Enc $\nabla Z$ into $\encb{\nabla Z}$\;
		\ForEach{$i \gets 1$ \KwTo $M$}
		{
			Send $\encb{\nabla Z}$\; 
			$\nabla \twai - \phi_\ai =\ $\HEtoSS{}\;
			Update $\vai$ via $\nabla \twai - \phi_\ai$, 
				Enc and Send $\vai$\;
		}
		Update $\ub$ via $\nabla \twb = \xb^T \nabla Z$
	}
	\uElse (\tcp*[h]{\hosti})
	{
		Recv $\encb{\nabla Z}$\; 
		$\phi_\ai =\ $\HEtoSS{$\xai^T\encb{\nabla Z}$}\;
		Recv $\encb{\vai}$
	}
}

\end{algorithm}

Although we focus on the two-party setting 
throughout this paper, 
the proposed federated source layers of \alg 
can also be generalized to the multi-party setting 
where there are two or more \host's. 
Here we provide an example of how to 
adapt the \texttt{MatMul} source layer 
to more than one \host's. 
The adaption for \texttt{Embed-MatMul} 
can be achieved similarly. 

The multi-party \texttt{MatMul} source layer 
is presented in Algorithm~\ref{alg:matmul_multi}. 
The idea is to secretly share the model weights of 
each \host with \guest one by one and 
let all \host's execute the same routines. 
To be specific, for the $i$-th \host (denoted as \hosti), 
the model weights are secretly shared as 
$\twai = \uai + \vai$, where $\vai$ is managed by \guest. 
For \guest, the model weights are broken into more pieces, 
i.e., $\twb = \ub + \sum_i \vbi$, where $\vbi$ is managed by \hosti.

\begin{figure*}[!t]
\begin{minipage}{.62\textwidth}
\centering
\small
\captionof{table}{{Averaged training time cost of one mini-batch (in seconds). We only record the time cost of matrix multiplication for a fair comparison.}}
{\begin{tabular}{|c|c|c|c|c|}
\hline
\multirow{3}*{\specialcell{\textbf{Dataset \&}\\\textbf{Sparsity}}}
& \multirow{3}*{\textbf{Model}} 
& \multicolumn{3}{c|}{\textbf{Time Cost/Batch (Seconds)}} 
\\
\cline{3-5}
& & \textbf{\alg} & \textbf{SecureML} 
& \specialcell{\textbf{SecureML}\\\textbf{(Client-aided)}}
\\
\hline
\hline
\textit{fmnist} (Dense) 
& MLP 
& 16.741
& 38.913 
& 0.006 
\\
\hline
\end{tabular}}
\label{tb:fmnist_efficiency}
\end{minipage}
\begin{minipage}{.01\textwidth}
$ $
\end{minipage}
\begin{minipage}{.35\textwidth}
\centering
\includegraphics[width=1.6in]{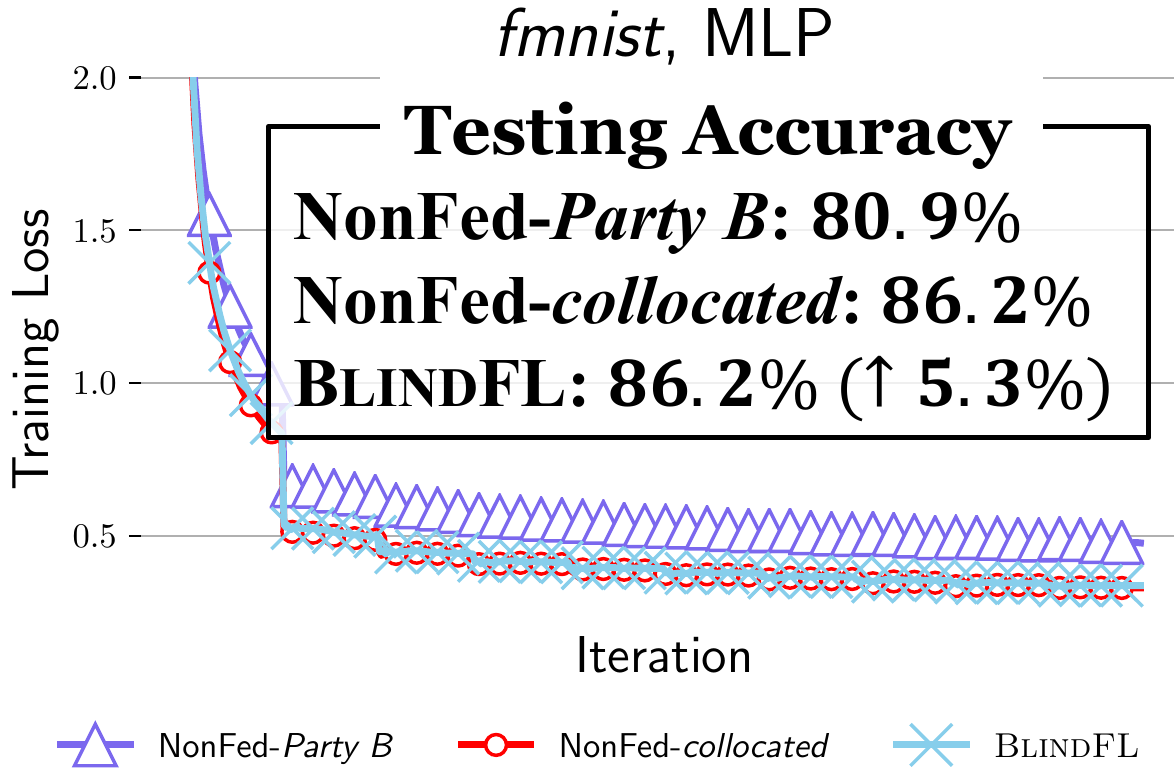}
\captionof{figure}{{Training loss in terms of iterations and testing AUC or accuracy metrics.}}
\label{fig:fmnist_converge}
\end{minipage}
\end{figure*}

\begin{figure*}[!t]
\begin{minipage}{.48\textwidth}
\small
\centering
\captionof{table}{{Scalability w.r.t. the output dimensionality of the \texttt{MatMul} source layer (\textit{connect-4}, 3-layer MLP).}}
{\begin{tabular}{|c|c|c|c|c|}
\hline
\textbf{Hidden Dim} & 32 & 64 & 128 & 256 \\
\hline
\hline
\textbf{Related Time Cost} 
& 1.00$\times$
& 1.91$\times$
& 3.94$\times$
& 8.06$\times$
\\
\hline
\textbf{Validation Accuracy} 
& 79.8\%
& 80.6\%
& 81.8\%
& 82.3\%
\\
\hline
\end{tabular}}
\label{tb:scalability_hidden} 
\end{minipage}
\begin{minipage}{.01\textwidth}
$ $
\end{minipage}
\begin{minipage}{.48\textwidth}
\small
\centering
\captionof{table}{{Scalability w.r.t. the number layers (\textit{connect-4}, MLP).}}
{\begin{tabular}{|c|c|c|c|c|}
\hline
\textbf{\# Layers} & 3 & 4 & 5 & 6 \\
\hline
\hline
\textbf{Related Time Cost} 
& 1.00$\times$
& 1.01$\times$
& 1.02$\times$
& 1.02$\times$
\\
\hline
\textbf{Validation Accuracy} 
& 80.6\%
& 80.9\%
& 81.0\%
& 80.6\%
\\
\hline
\end{tabular}}
\label{tb:scalability_layer}
\end{minipage}
\end{figure*}

Consequently, the goal of forward propagation is to compute 
the following results: 
\begin{equation*}
\begin{aligned}
Z &= \sum_i \xai\twai + \xb\twb \\
&= \sum_i \xai(\uai + \vai) + \xb(\ub + \sum_i \vbi) \\
&= \sum_i \large(\xai\uai + \xai\vai + \xb\frac{\ub}{M} + \xb\vbi \large), 
\end{aligned}
\end{equation*} 
where $M$ is the number of \host's. 
This can be reckoned as performing 
federated matrix multiplication 
between \guest and each \hosti, respectively, 
and summing up the $M$ intermediate results 
to achieve the final results. 
The detailed routines are presented 
in Line 12-19 of Algorithm~\ref{alg:matmul_multi}. 

For the backward propagation, 
\guest can still compute model gradients 
$\nabla\twb = \xb^T\nabla Z$ on its own. 
For \hosti, it only needs to achieve the secretly shared 
model gradients $\ss{\phi_\ai, \nabla\twai - \phi_\ai}$ 
via a similar routine as Figure~\ref{fig:matmul_routine}. 
Line 20-31 of Algorithm~\ref{alg:matmul_multi} 
shows the detailed procedure. 

Finally, it is worthy to note that 
here we only provide a simple method 
to extend \alg to the multi-party setting. 
It is interesting to explore 
how to optimize the algorithm protocols 
with smaller computation complexity 
and fewer communication rounds. 
We would like to leave them as our future work.

\section{More experiments}

\subsection{Experiments on Fashion MNIST}
In Section~\ref{sec:expr}, 
we conduct our experiments on tabular datasets. 
To evaluate whether our work also supports 
more kinds of datasets, 
we conduct experiments on 
the Fashion MNIST (\textit{fmnist}) dataset 
in this section. 
To be specific, we split each image into 
two 14$\times$28 subfigures 
to simulate the feature partitioning in VFL. 
Then, we train the MLP models 
over the partitioned dataset. 
As shown in Table~\ref{tb:fmnist_efficiency} and Figure~\ref{fig:fmnist_converge}, 
the results are consistent with 
those on other datasets. 
First, \alg runs faster than SecureML without client aided 
but slower than the client-aided SecureML. 
Second, \alg has a similar convergence as 
non-federated learning on collocated data 
(NonFed-\textit{collocated}), 
and achieves better performance than 
non-federated learning on the features of \guest only 
(NonFed-\guest), 
verifying the lossless property of our work.

\subsection{Scalability}
In this section, 
we train the MLP models on the \textit{connect-4} dataset 
to empirically evaluate the scalability of our work. 

\textbf{Scalability w.r.t. \#hidden dim.\xspace}
We first fix the number of hidden layers as 3 
and vary the output dimensionality of the first layer, 
i.e., the \texttt{MatMul} source layer. 
(The output dimensionalities of the second and third layers 
are fixed as 16 and 3, respectively.)
As shown in Table~\ref{tb:scalability_hidden}, 
when the first layer is wider, 
the validation accuracy increases slightly, 
whilst the training time cost increases almost proportionally. 
This is because the number of cryptography operations 
in the \texttt{MatMul} source layer 
also increases correspondingly. 

\textbf{Scalability w.r.t. \#layers.\xspace}
Then, we assess the effect of the number of layers. 
To do so, we fix the output dimensionalities of 
the first and the last but one layers as 64 and 16, 
and vary the number of layers by 
inserting 32-unit layers in the middle. 
For instance, in the 4-layer case, 
the output dimensionalities of all layers 
are 64, 32, 16, and 3, respectively. 
As shown in Table~\ref{tb:scalability_layer}, 
the number of layers does not affect 
the training time cost significantly. 
This is reasonable since the top model is 
a non-federated module 
and the major overhead is the \texttt{MatMul} source layer.